\documentclass[openany, amssymb, psamsfonts]{amsart}
\usepackage{mathrsfs,comment}
\usepackage{amssymb}
\usepackage[usenames,dvipsnames]{color}
\usepackage[normalem]{ulem}
\usepackage{url}
\usepackage[all,arc,2cell]{xy}
\UseAllTwocells
\usepackage{enumerate}
\usepackage{hyperref}  
\usepackage{pgfplots}
\usepackage{tikz}
\pgfplotsset{compat=1.17}
\hypersetup{%
  bookmarksnumbered=true,%
  bookmarks=true,%
  colorlinks=true,%
  linkcolor=blue,%
  citecolor=blue,%
  filecolor=blue,%
  menucolor=blue,%
  urlcolor=blue,%
  pdfnewwindow=true,%
  pdfstartview=FitBH}   
\usepackage{listings}
\usepackage{xcolor}
\usepackage{float}
\newtheorem{thm}{Theorem}[section]
\newtheorem{cor}{Corollary}[section]
\newtheorem{prop}{Proposition}[section]
\newtheorem{lem}{Lemma}[section]

\theoremstyle{definition}
\newtheorem{defn}{Definition}[section]

\newtheorem{rem}{Remark}[section]

\makeatletter
\let\c@obs=\c@thm
\let\c@cor=\c@thm
\let\c@prop=\c@thm
\let\c@lem=\c@thm
\let\c@prob=\c@thm
\let\c@con=\c@thm
\let\c@conj=\c@thm
\let\c@defn=\c@thm
\let\c@notn=\c@thm
\let\c@notns=\c@thm
\let\c@exmp=\c@thm
\let\c@ax=\c@thm
\let\c@pro=\c@thm
\let\c@ass=\c@thm
\let\c@warn=\c@thm
\let\c@rem=\c@thm
\let\c@sch=\c@thm
\let\c@equation\c@thm
\numberwithin{equation}{section}
\makeatother

\title{A Viscosity Semigroup Framework for Stable Image Reconstruction}
\author{Arina Oberoi}

\begin{document}

\maketitle

\begin{abstract}
Starting from the axiomatic formulation of scale-space theory, we develop a viscosity-solution framework for multiscale image representations arising from degenerate elliptic-parabolic partial differential equations. Rather than introducing a new semigroup theory, we work within the standard viscosity-solution setting, using comparison principles to obtain well-posedness, uniqueness, and contraction in the supremum norm. This perspective is used to motivate a hybrid reconstruction operator in which a learned inverse map is followed by a nonlinear diffusion evolution. At the continuous level, the diffusion operator satisfies non-expansiveness, which provides stability for the reconstruction process; this framework is then evaluated on a CT-based mesothelioma classification task, where it attains an AUC of \(0.875\) with negligible variation across epochs, while the baseline model acquires AUC values from \(0.49\) to \(0.80\) without a clear convergence pattern. These observations are consistent with the stabilizing role suggested by the discussed viscosity theory.
\end{abstract}

\tableofcontents

\section{Introduction}

Modern medical imaging pipelines rely on accurate signal reconstruction from noisy, incomplete, or corrupted measurements. In applications such as CT-based tumor detection, the reconstructed image is not just an intermediate output, but the object on which all downstream clinical decisions depend. As a result, it is natural to require that the reconstruction process be stable, in the sense that small perturbations in the input data do not produce large deviations in the output that are diagnostically misleading.

Contemporary approaches to image reconstruction often rely on learned operators, typically implemented as convolutional encoder-decoder architectures, acting as learned inverse maps \(u_{\text{out}} = G_\theta(u_{\text{in}})\) which transform corrupted measurements into usable representations for downstream tasks such as segmentation or classification. While these methods achieve strong empirical performance, they lack a key feature emphasized in the PDE and scale-space literature: provable stability of the reconstruction process. Purely learned mappings, while often effective in practice, do not in general provide structural guarantees on how perturbations in the input propagate through the system. Classical nonlinear diffusion models address this limitation by constructing image evolutions that satisfy strong structural properties, including order preservation, uniqueness, and contraction in the supremum norm. In medical imaging, such guarantees provide a rigorous notion of robustness: the evolution cannot amplify differences between inputs.

The central objective of this thesis is to formulate and study a hybrid reconstruction framework that combines the expressivity of learned models with the stability perspective provided by nonlinear diffusion theory. Rather than replacing the neural reconstructor, the network output is interpreted as an intermediate representation that serves as the initial condition for a controlled scale-space evolution governed, at the continuous level, by a degenerate elliptic parabolic equation in the viscosity sense. The resulting operator incorporates a diffusion stage which, at the level of the continuous model, satisfies a contraction property in the supremum norm which motivates (but does not by itself prove) its use as a stabilizing component in the reconstruction pipeline.

To motivate the need for stability, this work's appendix studies how representations evolve in neural networks, viewing them as pipelines that transform noisy inputs into task-relevant structures. Empirical analyses (presented largely in the appendix) show that while networks effectively reorganize information, they provide no guarantees on how perturbations propagate through intermediate representations. This analysis serves two purposes. First, it provides a baseline characterization of representation dynamics in both feedforward and sequential models. Second, it highlights a key limitation: while neural networks learn how to organize information, they do not provide guarantees on how perturbations propagate through intermediate representations.

The discussion then returns to the central application of CT-based mesothelioma classification to evaluate whether incorporating a nonlinear diffusion step improves robustness. In this setting, the diffusion process acts as a mathematically structured denoising mechanism, suppressing nuisance variability while ensuring that the relevant features are preserved.

\section{Viscosity Solutions and Nonlinear Semigroups}

In classical image processing and PDE theory, multiscale representations are modeled as nonlinear evolution processes satisfying strong stability and invariance properties. These properties lead naturally to the formulation of scale-space representations as semigroups generated by degenerate elliptic parabolic equations. This framework sets the mathematical foundation for the hybrid reconstruction operator studied here: a learned inverse map followed by a controlled nonlinear diffusion evolution. The results reviewed in this section establish stability, uniqueness, and contraction properties for the continuous PDE evolution, which are later used as a conceptual framework for discussing robustness in the neural-network reconstruction pipeline.

\subsection{Scale-Space Representations and Semigroup Structure}

Let bounded and uniformly continuous \(u_0 : \Omega \subset \mathbb{R}^2 \to \mathbb{R}\) be the original image.

\begin{defn}[Scale Space]

Let \(X\) be a function space over \(\Omega\) (for example \(X = L^2(\Omega)\) or \(C^\infty(\Omega)\)), and let \(u_0 \in X\). A \emph{scale space representation} of $u_0$ is a one-parameter family
\[
u : [0,\infty) \times \Omega \to \mathbb{R}
\]
such that for each \(t \ge 0\), the slice \(u(t,\cdot)\) belongs to \(X\).
\end{defn}

Define the operator family \(\{T_t\}_{t \ge 0}\) on \(X\) by \(T_t : X \to X, T_t u_0 := u(t,\cdot)\).  The resulting family of operators \(\{T_t\}_{t\ge0}\) satisfies the standard properties of a nonlinear evolution semigroup arising in
scale-space theory, including the identity property at \(t=0\), the semigroup property, and stability in the supremum norm, as described in the scale-space frameworks of van den Boomgaard and
Alvarez, Guichard, Lions, and Morel (AGLM).

\begin{enumerate}
    \item \textbf{Initial condition}

We require \(t=0\) such that
\[
u(0,x) = u_0(x) \text{for all } x \in \Omega
\]
which in operator notation is
\[
T_0 = \mathrm{Id}_X
\]
i.e. \(T_0 u_0 = u_0\) for all \(u_0 \in X\). 

\item \textbf{Causality (order preservation)}
If $u_0 \le v_0$, then
    \[
    T_t u_0 \le T_t v_0
     \text{for all } t \ge 0
    \]
This property is related to, but different from, the viscosity comparison principle introduced later.

\item \textbf{Recursivity (semigroup property)}

\[
T_s T_t = T_{s+t},  \forall s,t \ge 0
\]

This requires 
\[
T_s(T_t u_0) = T_{s+t} u_0, \forall u_0 \in X,\ \forall s,t \ge 0
\]
such that, for the scale space \(u(t,x)\),
we have \(u(s+t,\cdot) = T_{s+t} u_0 = T_s\big(T_t u_0\big),\). This means evolving from scale \(0\) to \(t\) and then from \(t\) to \(s+t\) is the same as evolving directly from \(0\) to \(s+t\).

\begin{proof}
Assume the evolution equation generates a unique solution for each initial condition (uniqueness is established more rigorously in Corollary 4.8 below via the comparison principle, used here to show the semigroup structure). Fix $t\ge 0$. Define shifted function
\[
v(s,x):=u(t+s,x),  s\ge 0,\ x\in\Omega
\]

By the chain rule,
\[
\partial_s v(s,x) = \partial_t u(t+s,x)
\]
Since $u$ satisfies the evolution equation, $v$ satisfies the same PDE with initial condition $v(0,\cdot) = T_t u_0$. By uniqueness of solutions
\[
v(s,\cdot) = T_s(T_t u_0)
\]
On the other hand, $v(s,\cdot) = u(t+s,\cdot) = T_{t+s} u_0$, hence
\[
T_s T_t = T_{s+t}
\]
\end{proof}

\end{enumerate}

\subsection{Existence of a PDE Generator and Second-Order Structure}

Under the scale-space axioms of Alvarez-Guichard-Lions-Morel, together with regularity, locality, and invariance assumptions, a multiscale representation can be characterized by a (possibly nonlinear) second-order evolution equation of the form
\[
\partial_t u(t,x)
+
F\big(x,u(t,x),\nabla u(t,x),\nabla^2 u(t,x)\big)
= 0, (t,x)\in (0,\infty)\times\Omega
\]

In this paper, the operator family \(\{T_t\}_{t\ge 0}\) is \emph{generated} by the equation if, for each initial condition \(u_0\in X\), the function \(u(t,x):=T_tu_0(x)\) is the viscosity solution of the PDE with initial condition \(u(0,x)=u_0(x)\). When uniqueness holds, this identifies \(\{T_t\}_{t\ge0}\) with the solution operator of the evolution equation. The operator \(F\) is then required to satisfy the following structural properties imposed by the AGLM framework:

\begin{enumerate}
\item \textbf{Degenerate ellipticity (ex: heat equation)}
For symmetric matrices \(X,Y\),
\[
X \le Y  \Rightarrow  F(x,r,p,X) \ge F(x,r,p,Y)
\]

Here \(Y-X\) is positive semidefinite, so \(X \le Y\) means that \(Y\) has more curvature than \(X\) (the function with Hessian \(Y\) is more convex). Increasing curvature increases \(\operatorname{tr}(X)\). For example, in the heat equation, the operator can be written as
\[
\partial_t u = \Delta u
\]

\[
\partial_t u - \Delta u = 0
\]

Comparing with AGLM (\(\partial_t u + F(x,u,\nabla u,\nabla^2 u) = 0\)), this shows
\[
F = -\Delta u = -\operatorname{tr}(X)
\]
This shows that if curvature increases, \(\operatorname{tr}(X)\) increases and therefore \(-\operatorname{tr}(X)\) decreases. Increasing curvature cannot increase the operator \(F\); it can only decrease it (or leave it unchanged). Consequently, the operator is degenerate elliptic, which guarantees the comparison principles required to prove uniqueness of viscosity solutions, as discussed by Evans and proved more rigorously later in this paper.

\item \textbf{Contrast/greyscale invariance}
For any strictly increasing function $\phi : \mathbb{R} \to \mathbb{R}$,
\[
T_t(\phi \circ u_0) = \phi \circ T_t(u_0)
\]
The evolution commutes with monotone reparametrizations of intensity: relabeling greyscale values before evolving gives the same result as evolving first and relabeling after. 

\item \textbf{Locality}

The operator is local such that \((\mathcal{F}u)(x) = F(x,u(x),\nabla u(x),\nabla^2 u(x))\) depends on the behavior of \(u\) only in a small neighborhood of \(x\).

\item \textbf{Translation invariance}

\[
F(x,r,p,X) = F(r,p,X)
\]

The operator is invariant under spatial translation such that shifting the spatial variable \(x\) by any vector doesn't change the form of \(F\). The same evolution rule applies everywhere in space.

\item \textbf{Rotation invariance (isotropy)}

For any orthogonal matrix \(Q\),
\[
F(r, Qp, QXQ^T) = F(r,p,X)
\]

The operator is \emph{isotropic} if rotating the coordinate system, therefore rotating the gradient \(p\) and Hessian \(X\), leaves \(F\) unchanged; the evolution governed by the PDE has no preferred spatial direction and treats all orientations equally. This definition of rotation invariance of operators in image processing and scale-space formulations is further discussed by Luengo Hendriks and van Vliet.
\end{enumerate}

\section{Breakdown of Classical and Weak Solutions}

Considering the nonlinear evolution PDE 
\[
\partial_t u + F(x,u,\nabla u,\nabla^2 u) = 0
\]
which comes up often in image processing and scale-space theory, classical solutions may be inadequate since a classical solution requires 
\[
u \in C^{1,2}((0,\infty)\times\Omega)
\]
such that all derivatives appearing in the PDE exist and the equation holds pointwise at every $(t,x)$. This corresponds to the usual classical idea of solving by computing all derivatives, plugging them into the equation, and requiring the equality to hold at every point.

However, even for very smooth initial data, under the nonlinear PDE considered here, solutions may develop nonsmooth features since nonlinear diffusion preserves edges, corresponding to jumps or sharp transitions in the gradient of $u$ and, in turn, aren't differentiable in the classical sense. So, though the equation is physically and mathematically meaningful, it cannot be treated using classical calculus.

The usual notion of weak solution is also not appropriate since the standard formulation relies on writing the PDE in divergence form and then integrating by parts to move derivatives from $u$ onto a smooth test function. This requires $u$ to have derivatives in the Sobolev sense (e.g.\ $u \in L^2$ and $\nabla u \in L^2$ locally), since derivatives of $u$ appear only under integrals and are paired with $\varphi$ or $\nabla \varphi$. This relies on the divergence structure $\nabla\cdot(\cdots)$ as integration by parts converts (distributional) second-order derivatives of $u$ into first-order derivatives of $u$ and derivatives of the test function $\varphi$ under the integral sign. The key step is the integration-by-parts identity, as discussed in Brezis:

\[
\int_{\Omega} \big(-\nabla \cdot A(x)\big)\,\varphi(x)\,dx
=
\int_{\Omega} A(x)\cdot\nabla\varphi(x)\,dx,
 \forall \varphi \in C_c^\infty(\Omega)
\]

This lowers the order of derivatives acting on the unknown function $u$ such that second-order derivatives appearing in divergence form are transferred onto the smooth test function $\varphi$. As such, this also produces a variational (integral) identity that is well defined under Sobolev regularity assumptions. This allows the PDE to be interpreted in the framework of Sobolev spaces, since derivatives of $u$ appear only in weak form, paired with $\varphi$ or $\nabla \varphi$ inside integrals. 

\begin{lem}[{Aubert-Kornprobst}]
\begin{enumerate}
    \item[(i)] A fully nonlinear equation may admit a natural (continuous) solution that is not $C^2$.
    \item[(ii)] The usual weak (variational) formulation is not naturally available, since the operator is \emph{not} in divergence form. 
\end{enumerate}
\end{lem}

\begin{proof}
Consider the fully nonlinear equation
\[
|\nabla u| = 1  \text{in } \Omega := (-1,1) \subset \mathbb{R}
\]
with boundary condition
\[
u(-1) = u(1) = 0
\]
In one dimension, this reads
\[
|u'(x)| = 1  \forall x \in (-1,1)
\]

Let \(u(x) := 1 - |x|,  x \in [-1,1]\).This function is chosen since a good candidate here would be one that satisfies the boundary conditions while having a nonsmooth feature. With the chosen function, $u$ is Lipschitz and piecewise $C^\infty$, but \(\exists\) a corner at $x = 0$.

\[
u'(x) =
\begin{cases}
1, & x < 0 \\[4pt]
-1, & x > 0
\end{cases}
\]

So, \(|u'(x)| = 1  \forall x \neq 0\) which tells us that the equation holds pointwise on $(-1,1) \setminus \{0\}$ and, since
\[
u(-1) = 1 - |-1| = 0
\quad \text{and} \quad
u(1) = 1 - |1| = 0
\]
the boundary condition is also satisfied. However, $u$ is not differentiable at $x = 0$, so the solution here cannot be a classical solution. This proves (i). Also, note that operator $u \mapsto |\nabla u|$ isn't in divergence form, as it cannot be written as
\[
-\nabla \cdot A(x,u,\nabla u)
\]
for some vector field $A$. To show this, assume for contradiction that there exists a vector field \(A(x,u,\nabla u) = (A_1,\dots,A_n)\) such that \(|\nabla u| = -\nabla \cdot A(x,u,\nabla u)\) 

\[
\nabla \cdot A(x,u,\nabla u)
= \sum_{i=1}^n \frac{\partial}{\partial x_i} A_i(x,u,\nabla u)
\]

\[
\frac{\partial}{\partial x_i} A_i(x,u,\nabla u)
=
\partial_{x_i}A_i
+
\partial_u A_i\, u_{x_i}
+
\sum_{j=1}^n \partial_{p_j}A_i\, u_{x_i x_j}
\]

\[
\nabla \cdot A(x,u,\nabla u)
=
\sum_{i,j=1}^n \partial_{p_j}A_i \, u_{x_i x_j}
+ \text{(lower-order terms)}
\]

The above expression necessarily contains second derivatives of $u$ (via Hessian $\nabla^2 u$). However, the operator
\[
F(\nabla u) = |\nabla u|
\]
depends only on first derivatives of $u$ and no second derivatives. So, if
\[
|\nabla u| = -\,\nabla \cdot A(x,u,\nabla u)
\]
Then the right-hand side would involve $\nabla^2 u$, while the left-hand side would not, showing that no such vector field $A$ can exist. Therefore, since the operator cannot be written in divergence form, this shows the standard weak formulation is not available.
\end{proof}

\subsection{Viscosity Solutions}

The function $u(x)=1-|x|$ is nevertheless workable within the viscosity sense, as at the corner $x=0$, the equation $|u'|=1$ is enforced via comparison with smooth test functions touching $u$ from above and below, rather than by requiring the classical derivative $u'(0)$ to exist. Before showing this, I will introduce viscosity solutions from its basic principles.

The key idea of viscosity solutions is to construct a solution for nonlinear PDEs that doesn't require differentiability of $u$ or rely on divergence form or integration by parts, while preserving the comparison principle and uniqueness. Consider nonlinear second–order PDEs of the form \[\partial_t u + F(x,u,\nabla u,\nabla^2 u)=0, (t,x)\in (0,T]\times\Omega\]

As noted earlier, a classical solution requires \(u\in C^{1,2}((0,T]\times\Omega)\) such that all derivatives exist, and the equation holds pointwise; however, natural solutions of nonlinear PDEs may develop singularities or non-smooth features. Therefore, interpreting the equation must be done without differentiating $u$ directly and instead comparing $u$ with smooth test functions. If a smooth function $\phi$ touches $u$ from above or below at a point, then $\phi$ acts as a local smooth approximation of $u$ near that point. This evaluates the PDE on $\phi$ instead of on $u$. 

\begin{defn}[Touching Points]

Let $u \in C((0,T]\times\Omega)$. A function $\phi \in C^2((0,T]\times\Omega)$ touches $u$ from above at $(t_0,x_0)$ if
\[u - \phi \text{ has a local maximum at } (t_0,x_0)\]

Likewise, function $\phi \in C^2((0,T]\times\Omega)$ touches $u$ from below at $(t_0,x_0)$ if \[u - \phi \text{ has a local minimum at } (t_0,x_0)\]
\end{defn}

For both, the test function $\phi$ is tangent to the graph of $u$ at the touching point.

\begin{defn}[Viscosity Subsolution]
A continuous function $u$ is a viscosity subsolution if for every $\phi\in C^2$
touching $u$ from above at $(t_0,x_0)$:
\(\partial_t\phi(t_0,x_0)
+ H(t_0,x_0,u(t_0,x_0),\nabla\phi(t_0,x_0),\nabla^2\phi(t_0,x_0))
\le 0\)
\end{defn}

\begin{defn}[Viscosity Supersolution]
A continuous function $u$ is a viscosity supersolution if for every
$\phi\in C^2$ touching $u$ from below:

\(\partial_t\phi(t_0,x_0) + H(t_0,x_0,u(t_0,x_0),\nabla\phi(t_0,x_0),\nabla^2\phi(t_0,x_0))
\ge 0\)
\end{defn}

\begin{defn}[Viscosity Solution]
A viscosity solution is a function that is both a subsolution and a supersolution.
\end{defn}

\begin{defn}[Degenerate Ellipticity (Nonlinear)]
As illustrated by Aubert-Kornprobst, the nonlinear PDEs which 
come up in scale-space analysis belong to the class of degenerate elliptic 
Hamilton-Jacobi equations \(\partial_t u + H(t,x,u,\nabla u,\nabla^2 u)=0\) where 
\(H:(0,T]\times\Omega\times\mathbb{R}\times\mathbb{R}^N\times S^N \to \mathbb{R}\) is 
the Hamiltonian and $S^N$ denotes the space of symmetric $N\times N$ matrices. The 
fundamental assumption is degenerate ellipticity:
\[
S \le S'  \Longrightarrow 
H(t,x,r,p,S)\ge H(t,x,r,p,S')
\]
As noted earlier in the PDE sense, here, monotonicity illustrates a stability principle 
as $H$ is non-increasing in its matrix argument: increasing $S$ can only decrease or 
leave unchanged $H$. This is the nonlinear analogue of ellipticity and shows comparison 
principles, from which uniqueness results central to viscosity theory follow. 
\end{defn}

\begin{prop} Let $H:\Omega\times\mathbb{R}\times\mathbb{R}^N\times S^N\to\mathbb{R}$ be continuous and
degenerate elliptic such that \(S\le S' \rightarrow H(x,r,p,S)\ge H(x,r,p,S')\). For $u\in C^2(\Omega)$ satisfying \(H\big(x,u(x),\nabla u(x),\nabla^2 u(x)\big)=0  \text{in }\Omega\), the following holds: if $\phi\in C^2(\Omega)$ and $x_0\in\Omega$ is a local maximum of $(u-\phi)$, then \(H\big(x_0,u(x_0),\nabla\phi(x_0),\) \(\nabla^2\phi(x_0)\big) \le 0\). If $x_0$ is a local minimum of $(u-\phi)$, then \(H\big(x_0,u(x_0),\nabla\phi(x_0),\) \(\nabla^2\phi(x_0)\big)\ge 0\).
\end{prop}

\begin{proof}
Assume $x_0$ is a local maximum of $(u-\phi)$. Then, by standard calculus, \(\nabla u(x_0)=\nabla\phi(x_0)\) and \(\nabla^2 u(x_0)\le \nabla^2\phi(x_0)\). By degenerate ellipticity,

\(H\big(x_0,u(x_0),\nabla u(x_0),\nabla^2\phi(x_0)\big)
\le H\big(x_0,u(x_0),\nabla u(x_0),\nabla^2 u(x_0)\big)\).

Using $H(x_0,u(x_0),\nabla u(x_0),\nabla^2 u(x_0))=0$ and $\nabla u(x_0)=\nabla\phi(x_0)$, 
\[
H\big(x_0,u(x_0),\nabla\phi(x_0),\nabla^2\phi(x_0)\big)\le 0
\]
WLOG, this applies to the local minimum case: if $x_0$ is a local minimum of $(u-\phi)$ then
$\nabla^2 u(x_0)\ge \nabla^2\phi(x_0)$, and degenerate ellipticity results in the reversed inequality.
\end{proof}

Through viscosity solutions, the PDE is enforced through inequalities evaluated on smooth test functions rather than on the possibly nonsmooth solution itself. Derivatives of $u$ never appear as all derivatives are taken on $\phi$, thus providing a strong notion of solution for nonlinear PDEs where classical and weak formulations fail.

\subsection{Comparison Principle and Uniqueness}

The definitions above provide a notion of solution for fully nonlinearnon-divergence form equations, but they don't yet guarantee that the resulting evolution is well posed. The key property that resolves this issues is the comparison principle, which shows viscosity solutions are order-preserving and lead directly to uniqueness and the well-posedness of the nonlinear evolution semigroup. Discussion here does not establish existence of viscosity solutions in full generality here; instead, it assumes existence under standard structural conditions on \(F\), as guaranteed by classical results in viscosity theory.

\begin{thm}[Comparison Principle]

Assume \(F : \Omega \times \mathbb{R} \times \mathbb{R}^n \times S^n \to \mathbb{R}\) is continuous, proper (non-decreasing in \(u\)), and degenerate elliptic, i.e. \(X \le Y \;\Rightarrow\; F(x,r,p,X) \ge F(x,r,p,Y)\). Let $u$ be a bounded viscosity subsolution and $v$ a bounded viscosity
supersolution of \(\partial_t w + F(x,w,\nabla w,\nabla^2 w)=0
 \text{ in } (0,T]\times\Omega\). If \(u(0,x)\le v(0,x) \text{for all }x\in\Omega,\) then \(u(t,x)\le v(t,x)\)
 \(\text{for all }(t,x)\in[0,T]\times\Omega\). 
\end{thm}

This, analogous to the maximum principle for the heat equation, demonstrating that the nonlinear evolution is order-preserving. If crossing were possible, there would be a first time when $u - v$ attains a positive maximum. At this first crossing point, the two surfaces would touch tangentially, so the spatial gradients must coincide, 
$\nabla u = \nabla v$, and the Hessians satisfy $\nabla^2 u \le \nabla^2 v$. By the 
degenerate ellipticity of $F$, increasing the matrix argument can only decrease $F$, so
\[
F(x,u,\nabla u,\nabla^2 u) \;\ge\; F(x,u,\nabla u,\nabla^2 v)
\]
This forces the PDE to act against any attempted crossing, preventing $u-v$ 
from becoming positive. Hence the ordering $u \le v$ is preserved and the solutions cannot cross.

Evans' version of the full argument uses the doubling of variable  method and the Crandall-Ishii lemma to handle nonsmooth functions.

\begin{proof} (Sketch of Evans')

Let $u$ be a bounded viscosity subsolution and $v$ a bounded viscosity supersolution of \(\partial_t w + F(x,w,\nabla w,\nabla^2 w)=0 \text{ in }(0,T]\times\Omega\)
where $F$ is continuous and degenerate elliptic. For simplicity, assume either $\Omega=\mathbb{R}^n$ or periodic boundary
conditions so that boundary effects do not arise. Assume for contradiction \(M:=\sup_{(t,x)\in[0,T]\times\Omega}\big(u(t,x)-v(t,x)\big)>0\). For $\varepsilon>0$, let
\[
\Phi_\varepsilon(t,x,s,y)
=
u(t,x)-v(s,y)
-\frac{|x-y|^2}{2\varepsilon}
-\frac{|t-s|^2}{2\varepsilon}
\]
Since $u,v$ are bounded and $[0,T]\times\overline{\Omega}$ is compact,
the upper semicontinuous function $\Phi_\varepsilon$ attains a maximum
at some point $(t_\varepsilon,x_\varepsilon,s_\varepsilon,y_\varepsilon)$. From the fact that it attains a maximum, for any fixed $(\bar t,\bar x)$,
\[
\Phi_\varepsilon(t_\varepsilon,x_\varepsilon,s_\varepsilon,y_\varepsilon)
\ge
\Phi_\varepsilon(\bar t,\bar x,\bar t,\bar x)
=
u(\bar t,\bar x)-v(\bar t,\bar x)
\]
Picking $(\bar t,\bar x)$ near a point where $u-v$ is close to $M$ gives
\[
u(t_\varepsilon,x_\varepsilon)-v(s_\varepsilon,y_\varepsilon)\ge M-o(1)
\]
However, boundedness of $u$ and $v$ tells us
\[
\frac{|x_\varepsilon-y_\varepsilon|^2}{2\varepsilon}+\frac{|t_\varepsilon-s_\varepsilon|^2}{2\varepsilon}
\le C
\]
for constant $C$ independent of $\varepsilon$. So, \(|x_\varepsilon-y_\varepsilon|\to 0,
|t_\varepsilon-s_\varepsilon|\to 0 \text{ as }\varepsilon\to 0\). Set
\[
p_\varepsilon:=\frac{x_\varepsilon-y_\varepsilon}{\varepsilon}, a_\varepsilon:=\frac{t_\varepsilon-s_\varepsilon}{\varepsilon}
\]
At the maximum point of $\Phi_\varepsilon$, the Crandall-Ishii lemma gives symmetric matrices
$X_\varepsilon,Y_\varepsilon\in S^n$ such that
\[
(a_\varepsilon,p_\varepsilon,X_\varepsilon)\in \overline{J}^{2,+}u(t_\varepsilon,x_\varepsilon),(a_\varepsilon,p_\varepsilon,Y_\varepsilon)\in \overline{J}^{2,-}v(s_\varepsilon,y_\varepsilon)
\]
and the matrices satisfy
\[
\begin{pmatrix}
X_\varepsilon & 0\\
0 & -Y_\varepsilon
\end{pmatrix}
\le
\frac{3}{\varepsilon}
\begin{pmatrix}
I & -I\\
-I & I
\end{pmatrix}
\]

Since $u$ is a subsolution and $v$ is a supersolution
\[
a_\varepsilon + F\big(x_\varepsilon,u(t_\varepsilon,x_\varepsilon),p_\varepsilon,X_\varepsilon\big)\le 0
\]
\[
a_\varepsilon + F\big(y_\varepsilon,v(s_\varepsilon,y_\varepsilon),p_\varepsilon,Y_\varepsilon\big)\ge 0
\]
By subtracting, \(F\big(x_\varepsilon,u(t_\varepsilon,x_\varepsilon),p_\varepsilon,X_\varepsilon\big)
-
F\big(y_\varepsilon,v(s_\varepsilon,y_\varepsilon),p_\varepsilon,Y_\varepsilon\big)
\le 0\). Using degenerate ellipticity (monotonicity in the Hessian) and continuity of $F$, \(x_\varepsilon-y_\varepsilon\to 0, t_\varepsilon-s_\varepsilon\to 0\), an inequality forcing \(M=\sup(u-v)\le 0\) is obtained in the limit $\varepsilon\to 0$. This contradicts $M>0$. So, $\sup_{[0,T]\times\Omega}(u-v)\le 0$.
\end{proof}

\subsection{From Viscosity Theory to Nonlinear Diffusion Semigroups}

As seen, once comparison holds, the entire nonlinear evolution becomes well-posed as comparison implies that viscosity solutions are unique.

\begin{cor}[Uniqueness of Viscosity Solutions]
Assume $F$ is continuous and degenerate elliptic. If $u$ and $v$ are viscosity solutions of \(\partial_t w + F(x,w,\nabla w,\nabla^2 w)=0\) with the same initial data $u(0,x)=v(0,x)$, then \(u(t,x)=v(t,x) \text{for all } (t,x)\in[0,T]\times\Omega\).
\end{cor}

\begin{proof}
Since $u$ and $v$ are viscosity solutions, each is both a subsolution and supersolution. Applying the comparison principle to $(u,v)$ gives \(u(t,x)\le v(t,x) \text{ for all }(t,x)\). Switching $u$ and $v$ and, again, applying the comparison principle,\(v(t,x)\le u(t,x) \forall (t,x)\). Combining the two results gives $u(t,x)=v(t,x)$, which proves uniqueness.
\end{proof}

Existence together with uniqueness yields a well-defined evolution operator \(T_t u_0 := u(t,\cdot), t \ge 0\). 

\begin{prop}[Supremum-norm contraction of the nonlinear semigroup]
Let $\{T_t\}_{t\ge0}$ denote the evolution operators generated by the
equation. Then for any initial data $u_0, v_0 \in X$,
\[
\|T_t u_0 - T_t v_0\|_{\infty}
\le
\|u_0 - v_0\|_{\infty},
t \ge 0
\]
This contraction property is a fundamental structural feature of
nonlinear semigroups generated by degenerate elliptic evolution
equations and follows from the general nonlinear semigroup framework
developed by Crandall and Liggett. Assume the equation is invariant under addition of constants and that $(T_t)_{t\ge 0}$ is order-preserving (comparison principle) and satisfies the semigroup properties. Then, for all bounded initial data $u_0,v_0$,
\(\|T_t u_0 - T_t v_0\|_{L^\infty(\Omega)}\) \(\le \|u_0 -v_0\|_{L^\infty(\Omega)} \forall t\ge 0\).
\end{prop}

\begin{proof}
Let \(M := \|u_0 - v_0\|_{L^\infty(\Omega)}\). By definition of the sup norm, \(v_0 - M \le u_0 \le v_0 + M  \text{in } \Omega\). Applying order preservation of $T_t$ gives
\[
T_t(v_0 - M) \le T_t(u_0) \le T_t(v_0 + M)
\]
Also, the evolution is invariant under vertical shifts
\[
T_t(v_0 \pm M) = T_t(v_0) \pm M
\]

Taken together, these give \(T_t v_0 - M \le T_t u_0 \le T_t v_0 + M
\) which tells us \(|T_t u_0 - T_t v_0| \le M  \text{in } \Omega\). Taking the supremum over $\Omega$ gives
\[
\|T_t u_0 - T_t v_0\|_{L^\infty(\Omega)}
\le M
=
\|u_0 - v_0\|_{L^\infty(\Omega)}
\]
\end{proof}

This shows that the continuous evolution operator \(T_t\) is non-expansive in the supremum norm: perturbations in the initial data are not amplified by the PDE evolution. The scale parameter \(t\) is therefore the time variable of a well-posed nonlinear evolution, rather than merely a formal smoothing parameter. Under the structural assumptions above, such a multiscale representation can be interpreted as a nonlinear diffusion semigroup acting on the initial image. In the later reconstruction pipeline, this continuous theory serves as motivation and mathematical context for the diffusion stage.

\begin{center}
\begin{tabular}{c}
Degenerate ellipticity \\[3pt]
$\Downarrow$ \\[3pt]
Comparison principle \\[3pt]
$\Downarrow$ \\[3pt]
Uniqueness of viscosity solutions \\[3pt]
$\Downarrow$ \\[3pt]
Stable evolution operator $T_t$ \\[3pt]
$\Downarrow$ \\[3pt]
Nonlinear diffusion semigroup
\end{tabular}
\end{center}

\section{From Representation Learning to Stable Reconstruction}

Two components have been established: a learned reconstruction operator 
\(G_\theta : X \to X\) and a nonlinear diffusion semigroup 
\(\{T_t\}_{t \ge 0}\) that is non-expansive in the supremum norm. 
These are combined through composition to obtain a reconstruction operator 
with controlled stability properties.

\begin{defn}[Hybrid Reconstruction Operator]
Let \(G_\theta : X \to X\) be a learned reconstruction map and let 
\(\{T_t\}_{t \ge 0}\) be a nonlinear diffusion semigroup generated by a 
degenerate elliptic parabolic equation. The hybrid operator is defined by
\[
\mathcal{H}_t(u_0) := T_t(G_\theta(u_0))
\]
\end{defn}

\begin{prop}[Stability of the hybrid operator]
Assume that \(T_t\) is contractive in \(L^\infty(\Omega)\) and that 
\(G_\theta\) is Lipschitz with constant \(L_\theta\). Then
\[
\|\mathcal{H}_t(u_0) - \mathcal{H}_t(v_0)\|_\infty
\le
L_\theta \|u_0 - v_0\|_\infty
\]
\end{prop}

\begin{proof}
\[
\|\mathcal{H}_t(u_0) - \mathcal{H}_t(v_0)\|_\infty
=
\|T_t(G_\theta(u_0)) - T_t(G_\theta(v_0))\|_\infty
\le
\|G_\theta(u_0) - G_\theta(v_0)\|_\infty
\le
L_\theta \|u_0 - v_0\|_\infty
\]
\end{proof}

The estimate shows that, at the level of the continuous semigroup, the diffusion does not amplify the sensitivity of the learned reconstruction. Any instability in the pipeline is therefore 
inherited entirely from the Lipschitz behavior of \(G_\theta\), while the 
diffusion stage preserves non-expansiveness.

\vspace{2mm}
\begin{rem}[Lipschitz Bound and Practical Stability]
The stability estimate of Proposition~3.2 bounds perturbation amplification 
through $H_t$ by the Lipschitz constant $L_\theta$ of the learned map $G_\theta$. 
For neural networks, $L_\theta$ can be bounded above by the product of per-layer 
spectral norms $\prod_k \|W_k\|_2$, computable via power iteration on each weight 
matrix. In our implementation, this quantity was not measured at evaluation time; 
a direct numerical verification of the bound therefore remains a direction for 
future work. 
\end{rem}
\vspace{2mm}

In inverse problems such as medical image reconstruction, the mapping \(u_{\mathrm{out}} = G_\theta(u_{\mathrm{in}})\) produces an image rather than an intermediate feature representation. 
Unlike the semigroup operators introduced earlier, \(G_\theta\) is not 
constructed to satisfy a comparison principle or any contractive property. However, the composition
\[
u_{\mathrm{out}} = T_t(G_\theta(u_{\mathrm{in}}))
\]
introduces a structurally controlled evolution. The semigroup \(T_t\) enforces 
order preservation and non-expansiveness, therefore suppressing perturbations 
while preserving large-scale geometric features.

\subsection{CT Reconstruction Pipeline and Network Architecture}

The reconstruction stage in our current pipeline is implemented as a 3D U-Net generator which takes as input a tensor of shape \((D,H,W,1)\) with default size \((3,512,512,1)\). The depth dimension \(D=3\) is interpreted as a stack of neighboring CT slices, while the network performs convolutional encoding and decoding with downsampling only in the directions \((H,W)\) with strides \((1,2,2)\). 

\begin{figure}[H]
    \centering
    \includegraphics[width=0.65\linewidth]{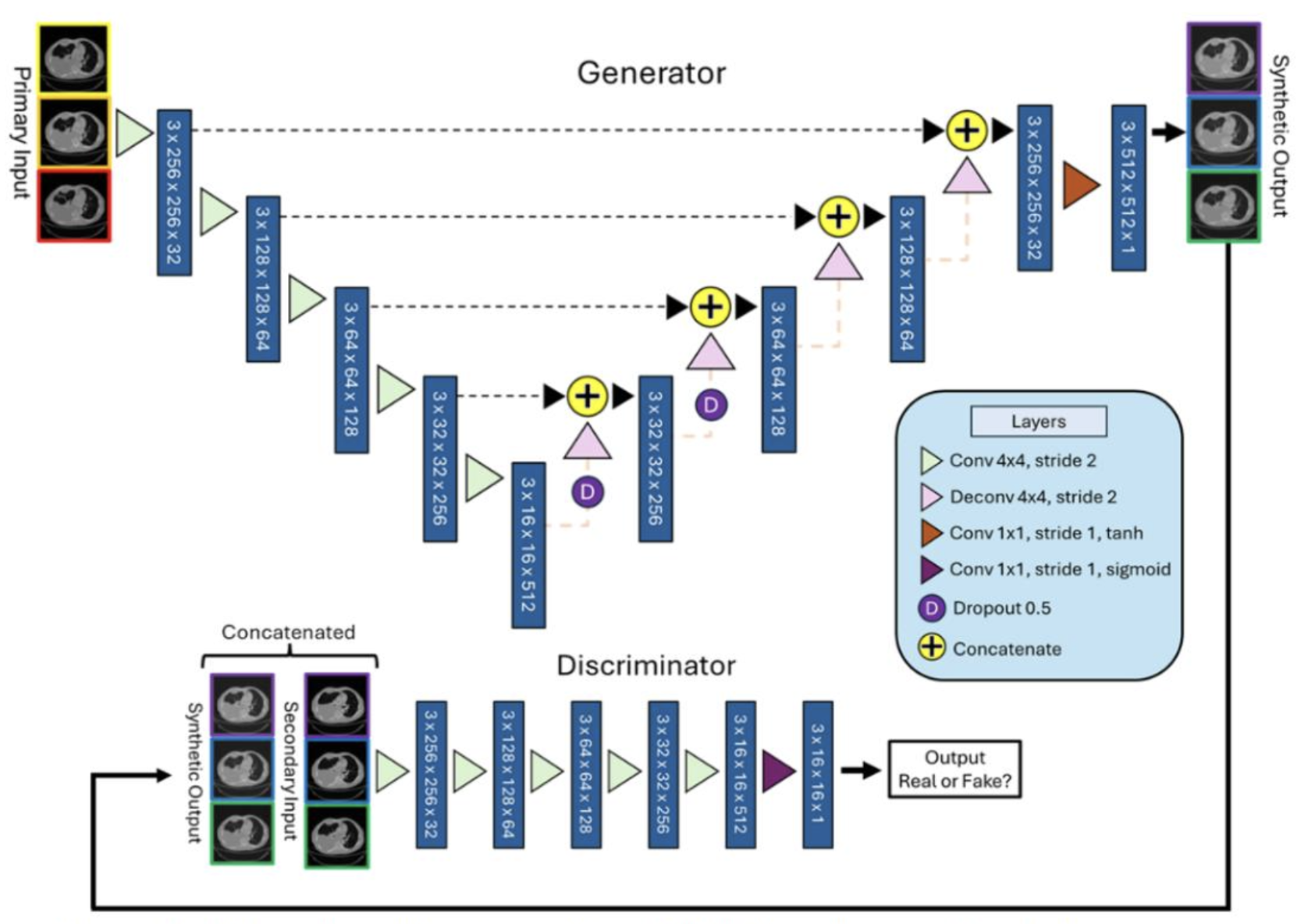}
    \caption{Network architecture; the shapes of each layer are depth, height, width, and channels.}
    \label{fig:placeholder}
\end{figure}

Specifically, the pipeline implements a learned inverse/denoising operator \(u_{\text{out}} = G_\theta(u_{\text{in}})\), where \(G_\theta\) is a neural network with parameters \(\theta\). Training seeks parameters \(\theta\) that minimize an empirical risk of the form 
\[\min_\theta \;\mathbb{E}_{(u_{\text{in}},u_{\text{gt}})}
\big[ \mathcal{L}(G_\theta(u_{\text{in}}),u_{\text{gt}}) \big]
\] 

This loss can be interpreted variationally as a combination of a fidelity and regularity term:
\[
\mathcal{L}(u,u_{\text{gt}})
=
\underbrace{\|u-u_{\text{gt}}\|_{L^2(\Omega)}^2}_{\text{data fidelity}}
\;+\;
\lambda\,\underbrace{\mathcal{R}(u)}_{\text{regularity}}
\;+\;
\gamma\,\mathcal{L}_{\text{adv}}(u)
\]

The fidelity term enforces agreement with ground truth reconstructions while the regularizer \(\mathcal{R}(u)\) encodes smoothness, edge preservation, or perceptual structure through network training. So, the generator can be interpreted as learning an implicit minimizer of an inverse problem of the form
\[
\min_u
\Big(
\|u - u_{\text{gt}}\|^2
+ \lambda \mathcal{R}(u)
\Big)
\]

Although effective, this reconstruction is produced by a single learned mapping and, therefore, doesn't inherit the stability and comparison principles available for evolution equations. So, our goal is not to replace the neural reconstructor, but to add a mathematically controlled PDE evolution that restores the scale-space and viscosity-solution structure developed above. The final reconstruction will be obtained by composing the learned operator with a nonlinear diffusion semigroup for improved CT image reconstruction, evaluated via AUC metrics. 

Although the generator $G_\theta$ is implemented as a neural network, the training objective has the form of a classical fidelity-plus-regularity energy. In deterministic inverse problems, such energies are known to produce nonlinear diffusion equations through their Euler-Lagrange optimality conditions. The purpose of this section is to show that the nonlinear diffusion operator introduced isn't arbitrary, but rather the natural first-order optimality condition associated with a reconstruction energy of fidelity-plus-regularity type.

\begin{prop}
The Euler-Lagrange equation of $E[u] = \frac{1}{2}\int_\Omega(u-s)^2\,dx + \int_\Omega\sqrt{1+|\nabla u|^2}\,dx$ is \(-\operatorname{div}\!\left(\frac{\nabla u}{\sqrt{1+|\nabla u|^2}}\right) + (u-s) = 0 \quad \text{in } \Omega\)
\end{prop}

\begin{proof}
Consider the energy functional
\[
E[u]
=
\frac12 \int_\Omega (u - s)^2\,dx
+
\int_\Omega \sqrt{1 + |\nabla u|^2}\,dx,
 \Omega = (0,1)^2
\]

The first term is a data fidelity term, enforcing agreement with the
observed image $s$, while the second term is a regularization term
allowing for spatial smoothness while preserving edges. Suppose $u$ is a minimizer of $E$. Then for every test function
$\varphi \in C_c^\infty(\Omega)$, consider perturbation \(u_\varepsilon = u + \varepsilon \varphi\). Let \(J(\varepsilon) := E[u_\varepsilon]\). Since $u$ is a minimizer, then \(J'(0) = 0\). The fidelity term is
\[
E_1[u] = \frac12 \int_\Omega (u - s)^2 dx
\]

For the shifted function,
\[
E_1[u_\varepsilon]
=
\frac12 \int_\Omega (u + \varepsilon \varphi - s)^2 dx
\]

\[
\frac{d}{d\varepsilon}E_1[u_\varepsilon]
=
\int_\Omega (u + \varepsilon \varphi - s)\,\varphi\,dx
\]

\[
\left.\frac{d}{d\varepsilon}E_1[u_\varepsilon]\right|_{\varepsilon=0}
=
\int_\Omega (u - s)\varphi\,dx
\]

The regularization term is
\[
E_2[u] = \int_\Omega \sqrt{1 + |\nabla u|^2}\,dx
\]

Again, for the shifted function,
\[
E_2[u_\varepsilon]
=
\int_\Omega \sqrt{1 + |\nabla u + \varepsilon \nabla\varphi|^2}\,dx
\]

Let \(h(\varepsilon) := 1 + |\nabla u + \varepsilon \nabla\varphi|^2\) \(\rightarrow\) \(E_2[u_\varepsilon] = \int_\Omega \sqrt{h(\varepsilon)}\,dx\). Using chain rule,
\[
\frac{d}{d\varepsilon}\sqrt{h(\varepsilon)}
=
\frac{1}{2\sqrt{h(\varepsilon)}} h'(\varepsilon)
\]

\(h'(\varepsilon)
=
2(\nabla u + \varepsilon \nabla\varphi)\cdot\nabla\varphi\) \(\rightarrow\) \(\frac{d}{d\varepsilon}\sqrt{h(\varepsilon)}
=
\frac{(\nabla u + \varepsilon \nabla\varphi)\cdot\nabla\varphi}
{\sqrt{1+|\nabla u + \varepsilon\nabla\varphi|^2}}\)

So, at $\varepsilon = 0$
\[
\left.\frac{d}{d\varepsilon}E_2[u_\varepsilon]\right|_{\varepsilon=0}
=
\int_\Omega
\frac{\nabla u\cdot\nabla\varphi}{\sqrt{1+|\nabla u|^2}}
\,dx
\]

Now, define vector field \(a(x) = \frac{\nabla u}{\sqrt{1+|\nabla u|^2}}\) such that \(\int_\Omega \frac{\nabla u\cdot\nabla\varphi}{\sqrt{1+|\nabla u|^2}} dx
=
\int_\Omega a \cdot \nabla\varphi\,dx\). Recall the divergence theorem: if $F$ is a smooth vector field on a bounded
domain $\Omega \subset \mathbb{R}^n$ with outward unit normal $n$, then
\[
\int_\Omega \operatorname{div} F \, dx
=
\int_{\partial\Omega} F\cdot n \, dS
\]

Applying this to $F = \varphi a$
\[
\operatorname{div}(\varphi a)
=
\nabla\varphi \cdot a + \varphi\,\operatorname{div} a
\]

\[
\int_\Omega \nabla\varphi \cdot a \, dx
+
\int_\Omega \varphi\,\operatorname{div} a \, dx
=
\int_{\partial\Omega} \varphi\, a\cdot n \, dS\]

Since the test function $\varphi \in C_c^\infty(\Omega)$ has compact support
inside $\Omega$, $\varphi=0$ on $\partial\Omega$ and the boundary term vanishes. This gives the integration-by-parts identity
\[
\int_\Omega a\cdot\nabla\varphi \, dx
=
- \int_\Omega (\operatorname{div} a)\,\varphi \, dx =
- \int_\Omega \operatorname{div}(a)\,\varphi\,dx
\]

Taking this with the definition of \(a(x)\) above
\[
\left.\frac{d}{d\varepsilon}E_2[u_\varepsilon]\right|_{\varepsilon=0}
=
- \int_\Omega
\operatorname{div}\!\left(
\frac{\nabla u}{\sqrt{1+|\nabla u|^2}}
\right)\varphi\,dx
\]

Finally, combining both derivatives,
\[
J'(0)
=
\int_\Omega (u - s)\varphi\,dx
-
\int_\Omega
\operatorname{div}\!\left(
\frac{\nabla u}{\sqrt{1+|\nabla u|^2}}
\right)\varphi\,dx
\]

Since $J'(0)=0$ for all $\varphi \in C_c^\infty(\Omega)$,
this gives the Euler-Lagrange equation
\[
- \operatorname{div}\!\left(
\frac{\nabla u}{\sqrt{1+|\nabla u|^2}}
\right)
+ (u - s)
= 0
 \text{in } \Omega
\]

Since this is a nonlinear diffusion equation, it shows that minimizing the reconstruction energy is equivalent to solving a nonlinear PDE driven by curvature-dependent diffusion.
\end{proof}

This variational derivation bridges the inverse–problem viewpoint and the nonlinear diffusion framework, showing that the diffusion equation stems from the continuous-time limit of fidelity–plus–regularity energy minimization.

\subsection{The Diffusion Trajectory as a Regularization Path}

The variational derivation of Proposition 3.3 shows that nonlinear diffusion equations of this type can be interpreted, at the continuous level, as gradient flows of fidelity-plus-regularity energies under appropriate structural assumptions. Rather than focusing solely on stationary points of such energies, it is natural to study the full evolution generated by the diffusion process.

To this end, consider the family of solutions
\[
\{u(\cdot, t)\}_{t \geq 0},
\]
parameterized by the scale variable $t$. This one-parameter family defines a trajectory in the function space $X = L^\infty(\Omega)$, with initial condition $u(\cdot, 0) = G_\theta(u_{\mathrm{in}})$ supplied by the learned reconstructor. This trajectory will be referred to as the \emph{regularization path} of the hybrid operator.

\begin{prop}[Monotone energy decay along the path (formal)]
Let $u(x,t)$ be a sufficiently smooth solution of
\[
\partial_t u + F(x, u, \nabla u, \nabla^2 u) = 0
\]
and suppose that $F$ is derived from the first variation of an energy functional
\[
E[u] = \frac{1}{2}\int_\Omega (u - s)^2\,dx + \int_\Omega \phi(|\nabla u|)\,dx
\]
with $\phi$ convex. Then, formally,
\[
\frac{d}{dt} E[u(\cdot, t)] \leq 0,
\]
so the energy is non-increasing along the trajectory.
\end{prop}

\begin{proof}
By the chain rule,
\[
\frac{d}{dt}E[u(\cdot,t)] = \int_\Omega \frac{\delta E}{\delta u} \, \partial_t u \, dx.
\]
If the evolution is given by the gradient flow
\[
\partial_t u = -\frac{\delta E}{\delta u},
\]
then
\[
\frac{d}{dt}E[u(\cdot,t)]
=
-\int_\Omega \left(\frac{\delta E}{\delta u}\right)^2 dx
\leq 0,
\]
with equality if and only if $u(\cdot,t)$ is a stationary point of $E$.
\end{proof}

This computation shows that, for diffusion equations admitting a gradient-flow interpretation, the solution trajectory is a descent curve in function space: each step along the path reduces the associated reconstruction energy. In the present setting, the Perona-Malik diffusion used in the implementation does not arise from a globally convex energy and is not uniformly parabolic. As a result, the above argument should be interpreted as a formal analogy rather than a rigorous characterization of the implemented flow.

\textbf{Bias-variance interpretation.}
The evolution can be interpreted through the competing effects of fidelity and regularity. For small $t$, the solution remains close to the initial reconstruction and retains fine-scale structure, including noise and reconstruction artifacts. For larger $t$, the diffusion increasingly suppresses spatial variation, leading to smoother outputs and potential loss of detail. The choice of a stopping time $T^*$, therefore, balances fidelity to the learned output against the stabilizing effect of the diffusion, linking the scale parameter $t$ to a model-selection problem.

\textbf{Connection to classical regularization.}
This perspective places the diffusion semigroup within the broader framework of regularization methods. In Tikhonov regularization,
\[
\min_u \|Au - b\|^2 + \lambda \mathcal{R}(u)
\]
the parameter $\lambda$ controls the strength of regularization. In the diffusion setting, the role of $\lambda$ is played implicitly by the stopping time $t$: increasing $t$ increases the effective regularization applied to the initial reconstruction. Analogous to the Morozov discrepancy principle, one may select a stopping time $T$ such that
\[
\|u(\cdot, T) - G_\theta(u_{\mathrm{in}})\|_{L^2(\Omega)} \approx \delta
\]
for a tolerance $\delta$ reflecting variability in the learned reconstruction.

\textbf{Implications for implementation.}
In the discrete scheme of Section 3.3, the parameters $K = 10$ and $\Delta t = 0.10$ determine an effective stopping time $T = K \cdot \Delta t = 1.0$. These parameters can, therefore, be interpreted as approximating a continuous stopping rule that balances fidelity and stability, even though the underlying diffusion is implemented through a finite, explicitly discretized evolution.

\subsection{First–Order Optimality for Variational Reconstruction Energies}

Many reconstruction objectives have the form \(\min_{u}\; \big( f(u) + g(u) \big)\) where $f$ is differentiable (fidelity) and $g$ may be nonsmooth (regularity), as discussed in Aubert and Kornprobst.

\begin{lem}[Subgradient optimality (see Brezis, Chapter 9)]
Let $f$ be convex and differentiable and let $g$ be convex. Then $u^\star$ minimizes $f+g$ if and only if \(0 \in \nabla f(u^\star) + \partial g(u^\star)\).
\end{lem}

\begin{proof}
A point $u^\star$ is a global minimizer of $f+g$ if and only if \(0 \in \partial (f+g)(u^\star)\). Since $f$ is convex and differentiable, \(\partial f(u^\star) = \{\nabla f(u^\star)\}\).  Recall the definition of the subdifferential:
\[
p\in\partial h(x)
\;\Longleftrightarrow\;
h(y)\ge h(x)+p\cdot (y-x) \forall y
\]

We want to show \(\partial (f+g)(x)=\nabla f(x)+\partial g(x)\). First, assume $p\in\partial (f+g)(x)$. Then, \(\forall\) $y$, \(f(y)+g(y)\ge f(x)+g(x)+p\cdot (y-x)\). Since $f$ is convex and differentiable, \(f(y)\ge f(x)+\nabla f(x)\cdot (y-x)\). Subtracting this gives \(g(y)\ge g(x)+\big(p-\nabla f(x)\big)\cdot (y-x)\) and, therefore, \(p-\nabla f(x)\in\partial g(x)\). This shows $p\in\nabla f(x)+\partial g(x)$. For the other direction, suppose $p=\nabla f(x)+q$ with $q\in\partial g(x)$. Then \(\forall\) $y$, \(f(y)\ge f(x)+\nabla f(x)\cdot (y-x), g(y)\ge g(x)+q\cdot (y-x)\). Adding them together, \(f(y)+g(y)\ge f(x)+g(x)+p\cdot (y-x)\). This shows $p\in\partial (f+g)(x)$. Therefore, for convex functions, the subdifferential sum rule shows
\[
\partial (f+g)(u^\star)
= \partial f(u^\star) + \partial g(u^\star)
= \{\nabla f(u^\star)\} + \partial g(u^\star)
\]

This shows \(0 \in \partial (f+g)(u^\star)
\Longleftrightarrow
0 \in \nabla f(u^\star) + \partial g(u^\star)\)
\end{proof}

This optimality condition is the nonsmooth analogue of the
Euler-Lagrange equation derived earlier. 

\subsection{Implementation}

The reconstruction network produces an initial estimate 
\[
u^{(0)} = G_\theta(u_{\mathrm{in}})
\]
To incorporate the scale-space and viscosity framework developed above, a classical edge-preserving diffusion model (Perona--Malik) is applied as a post-processing step. Specifically, a fixed number of explicit time steps ($K = 10$) of the equation
\[
\partial_t u \;=\; \nabla \cdot \!\big( g(|\nabla u|)\,\nabla u \big), 
\qquad
g(s) \;=\; \frac{1}{1+\bigl(\tfrac{s}{\kappa}\bigr)^2}
\]
are applied to the reconstructed image. The edge-stopping function $g$ suppresses diffusion at locations of large gradient magnitude while allowing smoothing in homogeneous regions. The refinement is applied post hoc to each reconstructed slice at evaluation time and is not used during training.

The Perona--Malik equation can be expressed within the general degenerate elliptic framework. Expanding the divergence term gives
\[
\partial_t u
=
g(|\nabla u|)\,\Delta u
+
\frac{g'(|\nabla u|)}{|\nabla u|}\,\nabla u^T \nabla^2 u\,\nabla u,
\]
which corresponds to an operator of the form
\[
F(p,X)
=
- g(|p|)\operatorname{tr}(X)
-
\frac{g'(|p|)}{|p|} p^T X p.
\]

\begin{prop}
For \(g(s)=\frac{1}{1+(s/\kappa)^2}\), the matrix
\[
A(p)=g(|p|)I+\frac{g'(|p|)}{|p|}\,p\otimes p
\]
is positive semidefinite whenever \(|p|<\kappa\). In particular, the operator is degenerate elliptic in the forward-diffusion regime.
\end{prop}

\begin{proof}
Let \(e = p/|p|\) for \(p \neq 0\). The matrix \(A(p)\) has eigenvalue \(g(|p|)\) in directions orthogonal to \(e\), and eigenvalue \(g(|p|) + |p| g'(|p|)\) in the direction \(e\). For \(g(s)=\frac{1}{1+(s/\kappa)^2}\),
\[
g'(s)= -\frac{2s/\kappa^2}{\left(1+(s/\kappa)^2\right)^2}
\]
\[
g(s)+s g'(s)
=
\frac{1-(s/\kappa)^2}{\left(1+(s/\kappa)^2\right)^2}
\]
Thus, \(g(s)\ge 0\) for all \(s\), and \(g(s)+s g'(s)\ge 0\) whenever \(s<\kappa\). All eigenvalues of \(A(p)\) are nonnegative when \(|p|<\kappa\), showing \(A(p)\) is positive semidefinite.
\end{proof}

This condition ensures degenerate ellipticity for the continuous model only in the forward-diffusion regime \(|p|<\kappa\), placing the equation locally within the class of PDEs for which viscosity-solution methods are relevant under appropriate assumptions. This does not amount to a global well-posedness result for the Perona-Malik flow, and the theoretical discussion here should, therefore, be interpreted as partial structural motivation rather than a complete justification of the implemented model.

The diffusion step is implemented using an explicit Euler discretization. For each iteration,
\[
u^{k+1}
=
u^{k}
+
\Delta t\,\nabla_h\cdot\!\big(g(|\nabla_h u^k|)\,\nabla_h u^k\big), k = 0,\dots,K-1
\]
where \(\nabla_h\) and \(\nabla_h\cdot\) denote finite-difference approximations of the gradient and divergence. Neumann (zero-flux) boundary conditions are imposed. All inputs are normalized to a fixed intensity range prior to processing so that the scale parameter \(\kappa\) has a consistent interpretation across samples. The resulting operator
\[
T_{\Delta t}(u)
:=
u + \Delta t\,\nabla_h\cdot\!\big(g(|\nabla_h u|)\,\nabla_h u\big)
\]
defines a discrete evolution, and the final reconstruction is
\[
u_{\mathrm{out}}
=
T_{\Delta t}^{\,K}\!\big(G_\theta(u_{\mathrm{in}})\big)
\]

While the continuous Perona-Malik equation is not uniformly parabolic and may exhibit ill-posed behavior, the discrete scheme used here is evaluated empirically rather than analyzed for numerical stability in the classical sense.

\subsection{A viscosity-semigroup interpretation of the reconstruction pipeline}

The network output \( u(0) := G_{\theta}(u_{\mathrm{in}}) \) is interpreted as an initial condition for a fixed, \(K\)-step controlled scale-space evolution. We construct the learned mapping using the order-preserving time-step operator \(T_{\Delta t}\) associated with a degenerate elliptic diffusion, and define the final reconstruction as
\[
u_{\mathrm{out}} = T_{\Delta t}^{K}\, G_{\theta}(u_{\mathrm{in}}), 
\quad K \in \mathbb{N}.
\]

\begin{thm}[Barles--Souganidis convergence of monotone schemes]
Consider a fully nonlinear, second-order PDE that is degenerate elliptic and interpreted in the viscosity sense. Suppose a numerical scheme is
\begin{itemize}
    \item monotone,
    \item stable, and
    \item consistent with the PDE.
\end{itemize}
Then the discrete solutions generated by the scheme converge uniformly, under refinement of the spatial grid and time step, to the unique viscosity solution of the continuous equation.
\end{thm}

This result provides a useful point of reference for numerical schemes designed for viscosity solutions. However, since monotonicity, consistency, and stability are not fully verified for the present discretization, the Barles-Souganidis theorem is cited here only as background rather than as a convergence result established for this implementation. In this thesis, it mainly serves to motivate the use of a short, fixed-length diffusion evolution as a structured refinement step for learned reconstructions, rather than as a theorem directly guaranteeing the observed empirical behavior.

The refinement is applied with parameters \(K = 10\), \(\Delta t = 0.10\), and \(\kappa = 0.08\). These parameters were selected to balance visible smoothing against edge preservation in the reported experiments. Because the timestep does not satisfy the classical CFL-type bound discussed later, the observed numerical behavior should be interpreted as an empirical property of this implementation rather than as a consequence of a proved stability theorem.

\section{Exploratory Testbed: Mesothelioma Subtype Classification on CT}

We compare the baseline generator with a hybrid reconstruction pipeline in which the network output is post-processed by a nonlinear diffusion step motivated by the viscosity-semigroup framework described above. Performance is evaluated by computing AUC on a fixed test set at each epoch, using the corresponding generator checkpoint. The goal of this section is exploratory: to assess whether the proposed post-processing is associated with improved empirical stability and test performance in this dataset.

\vspace{3mm}

\textbf{Baseline behavior.}
The baseline generator exhibits substantial variability across epochs. The observed AUC values range from approximately \(0.49 \;\; \text{to} \;\; 0.80\) with no clear convergence trend. While occasional peaks are achieved (e.g., AUC $\approx 0.796$ at epoch 22), these are not sustained, and the overall trajectory remains highly unstable. 

\begin{figure}[H]
    \centering
    \includegraphics[width=0.75\linewidth]{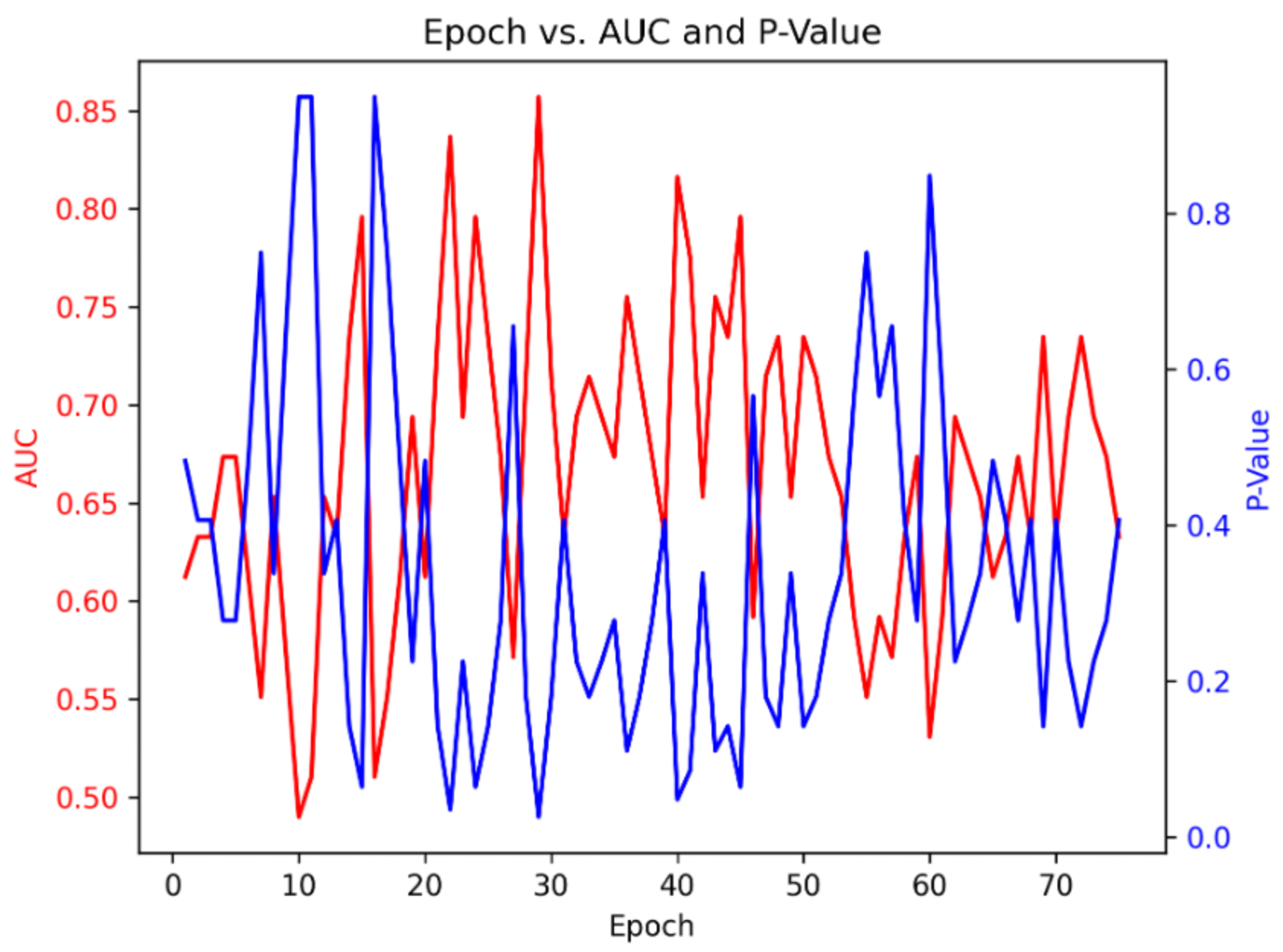}
    \label{fig:placeholder}
\end{figure}

\vspace{3mm}
\textbf{Viscosity-refined behavior.}
After introducing the nonlinear diffusion refinement, the observed performance profile changes markedly. Across all 80 epochs, the reported AUC is \(\text{AUC}_{\text{viscosity}} = 0.875\) with negligible variation.

\begin{figure}[H]
    \centering
    \includegraphics[width=0.75\linewidth]{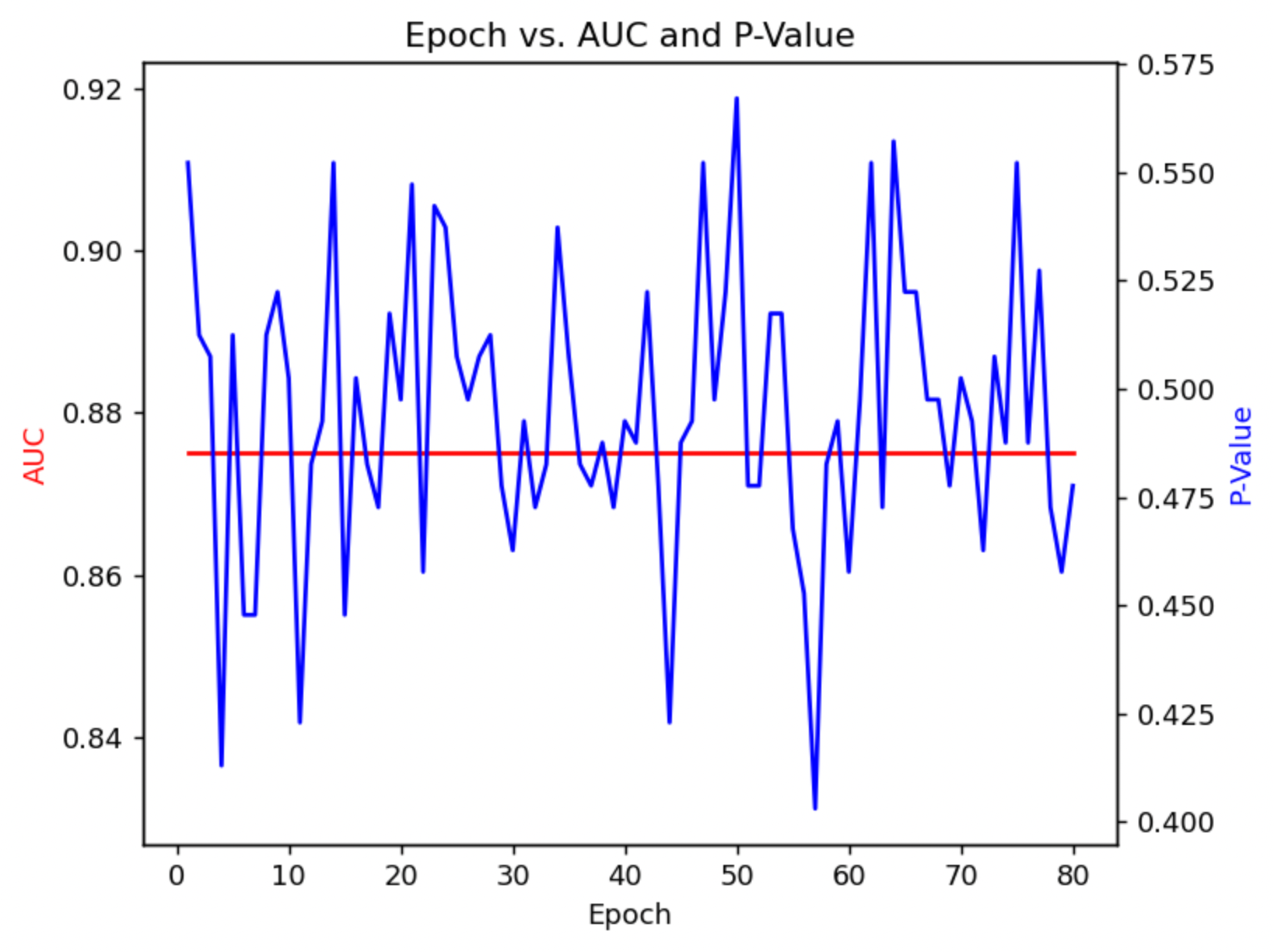}
    \label{fig:placeholder}
\end{figure}

\textbf{Performance comparison.}
Relative to the best observed baseline epoch, the refined pipeline shows a quantitative improvement: \(\Delta \text{AUC} = 0.875 - 0.7959 = 0.0791\), corresponding to a relative increase of approximately \(\frac{0.0791}{0.7959} \approx 9.9\%\); the reported epoch-to-epoch variability is also substantially smaller after refinement.

\textbf{Interpretation.}
The hybrid operator produces 
$u_{\mathrm{out}} = T_{\Delta t}^{\,K}(G_\theta(u_{\mathrm{in}}))$, where 
$T_{\Delta t}^{\,K}$ is a finite composition of explicit Euler steps 
motivated by the viscosity formulation of the Perona--Malik equation. 
By Proposition~2.11, the continuous semigroup $T_t$ is non-expansive in 
$L^\infty(\Omega)$. Under these operating conditions, the diffusion stage may reduce the effect of checkpoint-to-checkpoint variation in the outputs \(G_\theta(u_{\mathrm{in}})\). Note, however, that \(T_{\Delta t}^{\,K}\) is not the exact continuous semigroup, and this thesis does not prove that the discrete implementation inherits comparison or non-expansiveness in the precise sense available for the continuous PDE. 

The diffusion step, therefore, provides a structured, non-learned transformation of the generator output that is empirically associated with both improved AUC and reduced epoch-to-epoch variability in this dataset. While this behavior is consistent with the non-expansive character of the underlying continuous semigroup, the present results do not isolate a causal mechanism or rule out simpler smoothing effects. 

\subsection{Bootstrap Analysis of AUC Stability}

To quantify both the improvement in performance and the reduction in variability, a bootstrap analysis was performed on the epoch-wise AUC values. Let $\{a_1, \dots, a_{75}\}$ denote the baseline AUC values across training epochs.

\textbf{Bootstrap procedure.}
We generate $B$ bootstrap samples by sampling with replacement:
\[
\{a_1^{(b)}, \dots, a_{75}^{(b)}\}, \quad b = 1, \dots, B
\]
and compute the bootstrap mean
\[
\bar{a}^{(b)} = \frac{1}{75} \sum_{i=1}^{75} a_i^{(b)}
\]
The empirical distribution of $\{\bar{a}^{(b)}\}_{b=1}^B$ provides an estimate of the sampling distribution of the mean AUC.

\textbf{Baseline statistics.}
From the observed data,
\[
\bar{a}_{\text{baseline}} \approx 0.622, \min a_i = 0.490, \max a_i = 0.796
\]
The large spread indicates substantial variability across epochs. Let $\hat{\sigma}_{\text{baseline}}^2$ denote the sample variance:
\[
\hat{\sigma}_{\text{baseline}}^2
=
\frac{1}{74}
\sum_{i=1}^{75}
(a_i - \bar{a}_{\text{baseline}})^2
\]

A $95\%$ bootstrap confidence interval for the mean AUC is
\[
\left[
\bar{a}^{(0.025)}, \bar{a}^{(0.975)}
\right]
\approx [0.606,\;0.638]
\]

\textbf{Viscosity-refined statistics.}
After viscosity diffusion, the AUC is constant across epochs:
\[
a_i^{\text{viscosity}} = 0.875, \quad \forall i
\]
\[
\bar{a}_{\text{viscosity}} = 0.875,
\hat{\sigma}_{\text{viscosity}}^2 = 0
\]

Under bootstrap resampling, every sample is identical, so \(\bar{a}^{(b)}_{\text{viscosity}} = 0.875 \forall b\) and the confidence interval collapses to a point: \([0.875, \; 0.875]\).

\textbf{Comparison.}
The bootstrap analysis highlights two effects:

\begin{enumerate}
\item{Shift in mean performance:}
\[
\Delta \bar{a} = 0.875 - 0.622 \approx 0.253
\]

\item{Collapse of variability:}
\[
\hat{\sigma}_{\text{baseline}}^2 > 0
\quad \longrightarrow \quad
\hat{\sigma}_{\text{viscosity}}^2 = 0
\]
\end{enumerate}

The confidence intervals are disjoint, indicating that the observed improvement
in AUC is large relative to the baseline variability across epochs.

\textbf{Interpretation.}
The diffusion operator does not just improve expected performance, but also affects the distribution of outcomes; the mapping from input to reconstruction transitions from a high-variance estimator to a deterministic operator at the level of downstream metrics. This behavior is consistent with the contraction properties established earlier: the diffusion semigroup appears to suppress perturbations in the reconstructed image, suggesting that variability in the generator is not reflected in the final classification output.

\subsection{Conclusion.}

The results support the central claim of this work: incorporating a mathematically grounded diffusion stage into a learned reconstruction pipeline affects both the accuracy and the stability of downstream performance. While the baseline generator achieves moderate peak performance, it exhibits significant variability across training epochs, indicating sensitivity to training dynamics and a lack of structural control. In contrast, the hybrid operator produces uniformly high AUC values across all epochs.

This combination of higher AUC and near-zero variance is the key finding. It reflects not just an improvement in peak performance, but a shift from an unstable mapping to a controlled operator. Consistent with the theoretical framework, the diffusion semigroup appears to suppress perturbations introduced by the learned reconstruction, reducing sensitivity to the specific generator checkpoint and stabilizing downstream performance. The post-processing step is consistent with a non-expansive evolution, reducing variability while preserving diagnostically relevant structure.

Though there are clear limitations to the experimental setting, as discussed below, the results nonetheless demonstrate the value of integrating mathematical structure with data-driven models since the improvement observed here does not arise from increased architectural complexity or additional training, but from enforcing a mathematically well-posed evolution on the learned reconstruction.

\subsubsection{Limitations and Future Directions}

Several limitations should be acknowledged. First, the theoretical stability
estimate of Proposition~3.2 bounds perturbation amplification through the
hybrid operator in terms of the Lipschitz constant $L_\theta$ of the learned
map $G_\theta$, but this constant is never measured or controlled in our
implementation. The empirical evaluation measures downstream AUC rather than
input perturbation amplification directly, so the two are not immediately
comparable. The AUC improvement and elimination of epoch-to-epoch variability
are consistent with the contraction property of $T_{\Delta t}^K$, but a direct
numerical verification of the Lipschitz bound remains a direction for future work.

Second, the Perona--Malik diffusion model is not uniformly parabolic and is known to exhibit ill-posed behavior in the continuous setting. While the discretized scheme used here is stable in practice for small timesteps and limited iterations, the theoretical guarantees of viscosity solution theory do not directly apply without additional regularization. 

Third, the dataset of 66 patients is small for a deep learning pipeline, and
the stability of AUC at $0.875$ may partly reflect the low variance of the
diffusion post-processor rather than a genuine improvement in discriminative
power. A larger cohort would allow for more reliable cross-validation and external test set
evaluation.

Fourth, the diffusion parameters ($K = 10$, $\Delta t = 0.10$, $\kappa = 0.08$)
were fixed by manual tuning rather than optimized jointly with the generator.
Learning these parameters end-to-end, subject to the constraint that the
resulting operator remain monotone and non-expansive, is a natural extension.
One principled approach would be to parametrize $\kappa$ and $\Delta t$ as
learnable scalars and enforce the CFL stability condition as a differentiable
constraint during training.

Fifth, the comparison is limited to a single baseline. A more complete
evaluation would include comparisons against other stabilization strategies (spectral normalization, Tikhonov regularization, total variation post-processing, or plug-and-play priors) to isolate how much of the observed gain is attributable to the viscosity structure specifically, as opposed to any form of post-hoc smoothing.

\section*{Acknowledgments}
I would like to thank my advisor, Jiaqi Li, for her guidance and support throughout this project, as well as Christopher Valdes and Samuel Armato III for their collaboration on the CT reconstruction and detection framework, developed in the Armato Laboratory at The University of Chicago.

\newpage

This appendix contains supplementary experiments analyzing how representations evolve in learned models trained via stochastic optimization. These experiments are not part of the CT reconstruction pipeline itself, but provide empirical motivation for the central modeling choice of this work since they illustrate that learned mappings can exhibit variability and sensitivity across training iterations, even when overall task performance appears stable. This behavior illustrates an important limitation: standard neural network pipelines do not enforce control over the propagation of perturbations through intermediate representations. The instability observed here motivates the introduction of the diffusion-based operator in the main text, which enforces a controlled, non-expansive evolution. The final appendix provides the implementation details and code used in the CT experiments.

\appendix

\section{Statistical Considerations for Stochastic Training}

The learned components in this thesis, including the feedforward, reinforcement-learning, and CT reconstruction networks, are trained using stochastic optimization procedures, most notably stochastic gradient descent (SGD). As a result, estimates of model parameters and derived quantities such as mutual information proxies, accuracies, and AUC values depend on optimization randomness and finite-sample variation.

A framework for uncertainty quantification in stochastic optimization, following Zhu et al., provides a statistical basis for interpreting these results. Formally, the parameter of interest is the minimizer of the population loss:
\[x^* = \arg\min_{x \in \mathbb{R}^d} F(x), 
\quad \text{where } 
F(x) = \mathbb{E}_{\xi \sim \Pi}[f(x, \xi)]\]
In practice, this minimizer cannot be computed directly due to data scale and stochasticity. Instead, SGD produces a sequence of iterates \(x_1, x_2, \ldots\), each formed using noisy gradient estimates. These iterates converge toward \(x^*\), but exhibit variability due to sampling noise. Beyond point estimation, the goal is to quantify uncertainty in functionals of the solution, such as \(\nu^\top x^*\), through confidence intervals satisfying
\[
P(\nu^\top x^* \in \widehat{CI}) \approx 1-\alpha
\]

At high confidence levels, even small deviations from nominal coverage become significant. To measure this, Zhu et al. introduce the relative error
\[\Delta_\alpha = \left| \frac{P(\nu^\top x^* \notin \widehat{CI})}{\alpha} - 1 \right|\]
This quantity compares the actual error rate with the target level \(\alpha\). When \(\Delta_\alpha \approx 0\), the interval achieves correct coverage. Controlling this error is particularly important in settings involving multiple comparisons, where small miscalibrations accumulate. To construct confidence intervals, consider the averaged SGD (ASGD) estimator
\[
\bar x_n = \frac{1}{n}\sum_{i=1}^n x_i.
\]
The Polyak–Ruppert theorem establishes the asymptotic normality
\[
\sqrt{n}(\bar{x}_n - x^*) \Rightarrow \mathcal{N}(0, \Sigma)
\]
which implies that \(\bar x_n\) behaves approximately as a Gaussian estimator for large \(n\). This leads to confidence intervals of the form
\[
\nu^\top \bar x_n \pm z_{1-\alpha/2}\sqrt{\nu^\top \Sigma \nu / n}
\]

However, the covariance \(\Sigma = A^{-1} S A^{-1}\) is unknown in practice. Direct estimation is computationally expensive or unstable in high-confidence regimes. To address this, Zhu et al. propose a parallelized approach based on multiple independent runs. Specifically, \(K\) independent SGD trajectories are computed:
\[\hat x^{(k)}_i = h_i(\xi^{(k)}_i, \mathcal{F}^{(k)}_{i-1}), \quad k=1,\dots,K\] and averaged within each run. The final estimator is
\[\bar x_{K,n} = \frac{1}{K}\sum_{k=1}^K \hat x^{(k)}_n\]
The between-run variance is
\[\hat\sigma_\nu^2 = \frac{1}{K-1}\sum_{k=1}^K \big(\nu^\top \hat x^{(k)}_n - \nu^\top \bar x_{K,n}\big)^2\]
This leads to the studentized statistic
\[\hat t_\nu = \frac{\sqrt{K}\,(\nu^\top \bar x_{K,n} - \nu^\top x^*)}{\hat\sigma_\nu}\]
And, therefore, the confidence interval
\[\widehat{CI}_\nu = 
\Big[\nu^\top \bar x_{K,n} \ \pm\ t_{1-\alpha/2,\,K-1}\,\hat\sigma_\nu/\sqrt{K}\Big]\]

Because the runs are independent, standard $t$-distribution theory applies, avoiding the need for Hessian estimation or resampling. This method is computationally efficient, effective even for small $K$, and provides finite-sample guarantees. Instead of relying solely on asymptotic coverage,
\[
P(\nu^\top x^* \in \widehat{CI}) \to 1-\alpha
\]
the analysis controls relative error uniformly:
\[
\Delta_N := \sup_{\alpha(N)\le \alpha < 1} \left|\frac{P(\nu^\top x^* \in \widehat{CI}) - (1-\alpha)}{\alpha}\right| \to 0
\]

Under regularity assumptions, the approximation error satisfies
\[
\mathbb{E}\|W_n - Z_n\|^2 \ \lesssim \ \max\!\left\{n^{1-2\beta},\ \frac{\log n}{n^{1-2/q}},\ \frac{\|x_0 - x^{*}\|^2}{n}\right\}
\]
and the studentized statistic converges at rate
\[
\sup_{z} \big|P(\hat t_\nu \ge z) - P(t_{K-1} \ge z)\big| \ \lesssim\ \delta(N/K)^{1/4}
\]

These results provide a statistical foundation for interpreting the variability observed in stochastic training. In particular, they show that variability across training runs and epochs is an inherent feature of stochastic optimization, rather than an artifact of implementation. 

\section{Information Flow in Feedforward Networks}

Feedforward architectures provide a controlled setting in which representation dynamics can be measured directly. The objective is not to optimize benchmark performance, but to characterize how information is preserved, compressed, and reorganized across layers in the absence of explicit stability constraints.

Experiments are conducted on the MNIST dataset, consisting of $70{,}000$ grayscale images of size $28 \times 28$. Each image is normalized to $[0,1]$ and models are trained using stochastic gradient descent (SGD). At initialization, predictions are effectively random. Each SGD iteration processes a mini-batch, computes the loss, and updates parameters. Over time, this produces a sequence of representations that progressively encode task-relevant structure. Three model classes are considered:

\subsubsection{Linear Model}

A linear classifier trained with SGD (learning rate $0.1$, $5$ epochs) achieves a final test accuracy of $91.62\%$. The training loss decreases from $0.3745$ to $0.2789$, with stable convergence and minimal variance across epochs. The model quickly reaches its capacity, reflected in early plateauing of the loss.

\subsubsection{Multi-Layer Perceptron (MLP)}

An MLP trained under the same conditions achieves $97.62\%$ test accuracy. The loss decreases from $0.3189$ to $0.0667$, with rapid improvement in early epochs followed by gradual refinement. This reflects the network’s ability to extract higher-level features beyond linear separability.

\subsubsection{Convolutional Neural Network (CNN)}

The CNN achieves $98.93\%$ test accuracy, with loss decreasing from $0.2764$ to $0.0264$. The improved performance reflects hierarchical feature extraction, where spatial structure is progressively encoded across layers.

\subsection{Measuring Information Flow}

To quantify representation dynamics, I use empirical proxies for mutual information to measure how each layer represents the input $X$ and the task label $Y$:
\[
\widehat I(X; T_l), \quad \widehat I(Y; T_l)
\]

Here, $\widehat I(X;T_l)$ is interpreted as a proxy for input preservation, while $\widehat I(Y;T_l)$ is interpreted as a proxy for the accessibility of task-relevant structure under the chosen estimator. These quantities should not be read as exact mutual information values. Activations are recorded across layers:
\[
X \rightarrow T_1 \rightarrow T_2 \rightarrow \cdots \rightarrow T_L
\]

Mutual information is defined as
\[
I(U; V) = \sum_{u,v} p(u,v)\log \frac{p(u,v)}{p(u)p(v)}
\]

However, exact mutual information is not computed here. In a deterministic feedforward network, the data processing inequality implies that true mutual information cannot increase as information passes through successive layers:

\[
I(Y;T_{l+1}) \leq I(Y;T_l)
\]

Therefore, any apparent increase in the plotted quantity should be interpreted as an increase in the estimated or proxy mutual information, not as an actual increase in true information. In this appendix, increases in $\widehat I(Y;T_l)$ indicate that task-relevant structure has become more accessible to the classifier after the network reorganizes the representation. Thus, the plots are best understood as diagnostics of representation organization rather than as theorems about exact information flow.

\subsection{Information Flow Analysis}

For the linear model, only a single layer is present. $\widehat I(X; T_l)$ is low, indicating early compression of input structure under the proxy estimator, while $\widehat I(Y; T_l)$ increases modestly, suggesting that what little task-relevant structure exists becomes more accessible to the estimator. This reflects limited representational capacity, not an increase in true mutual information.

\begin{figure}[H]
    \centering
    \includegraphics[width=0.75\linewidth]{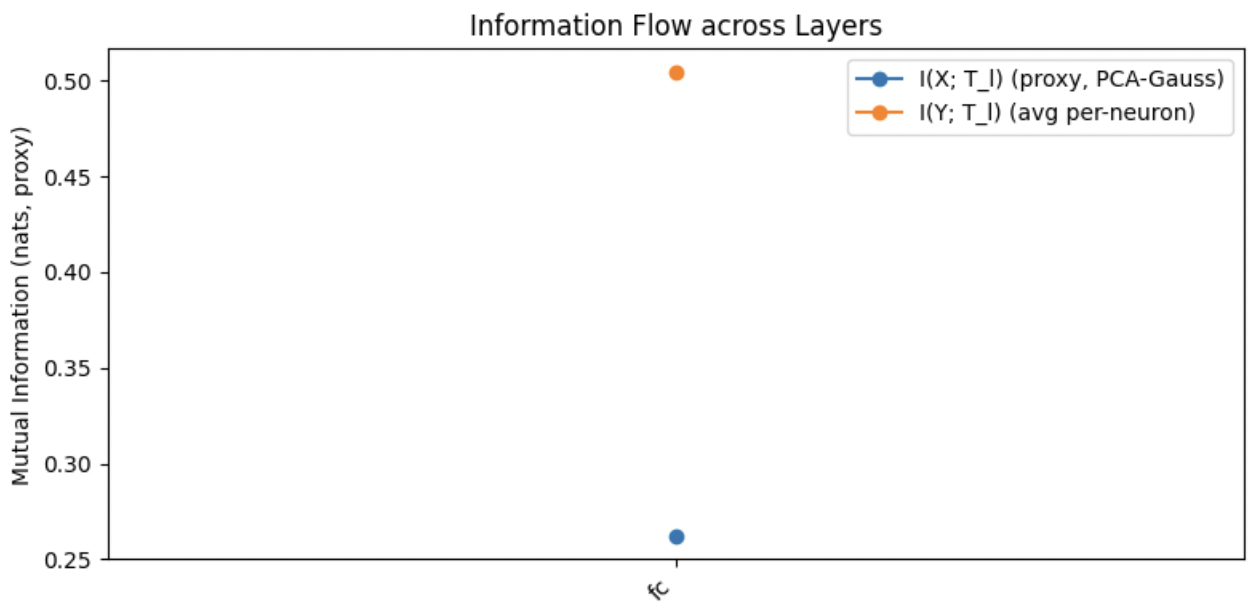}
    \caption{Information flow across layers for the linear model.}
    \label{fig:placeholder}
\end{figure}

For the MLP, $\widehat I(X; T_l)$ remains relatively low and stable, while $\widehat I(Y; T_l)$ increases across layers. This should be interpreted as the task-relevant structure becoming more accessible under the chosen proxy, rather than as an increase in true mutual information.

\begin{figure}[H]
    \centering
    \includegraphics[width=0.75\linewidth]{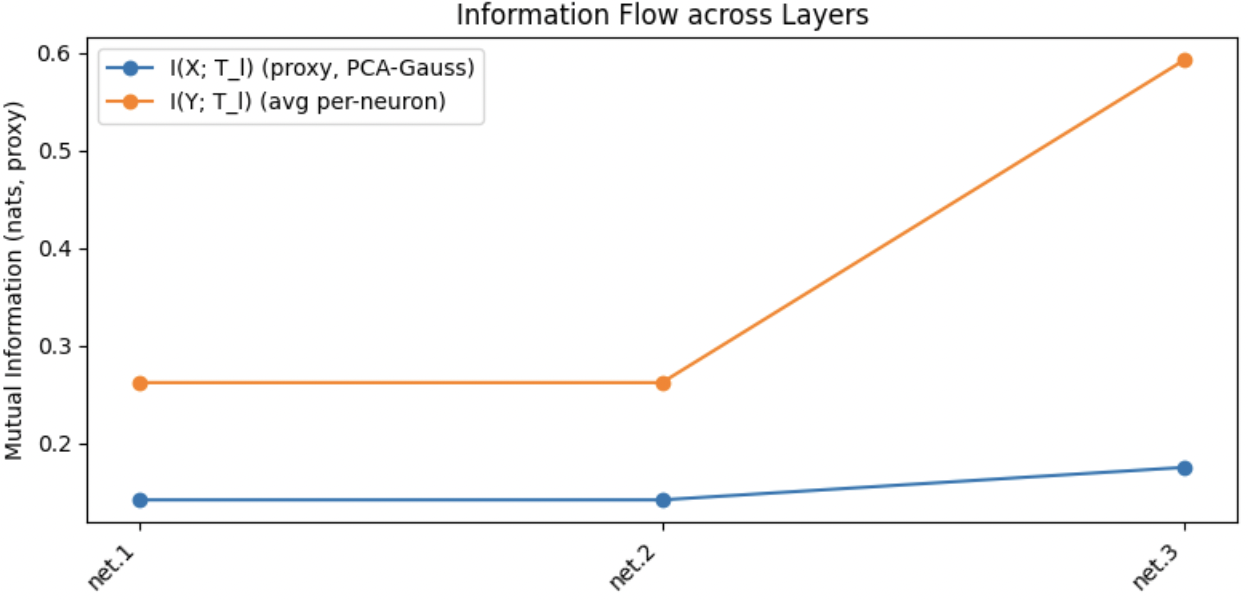}
    \caption{Information flow across layers for the MLP.}
\end{figure}

For the CNN, early layers maintain high estimated input preservation and low estimated task-label accessibility, while deeper layers exhibit stronger compression of input structure and a larger proxy value $\widehat I(Y;T_l)$. This reflects the reorganization of features into a form more useful for classification, not the creation of new information.

\begin{figure}[H]
    \includegraphics[width=0.75\linewidth]{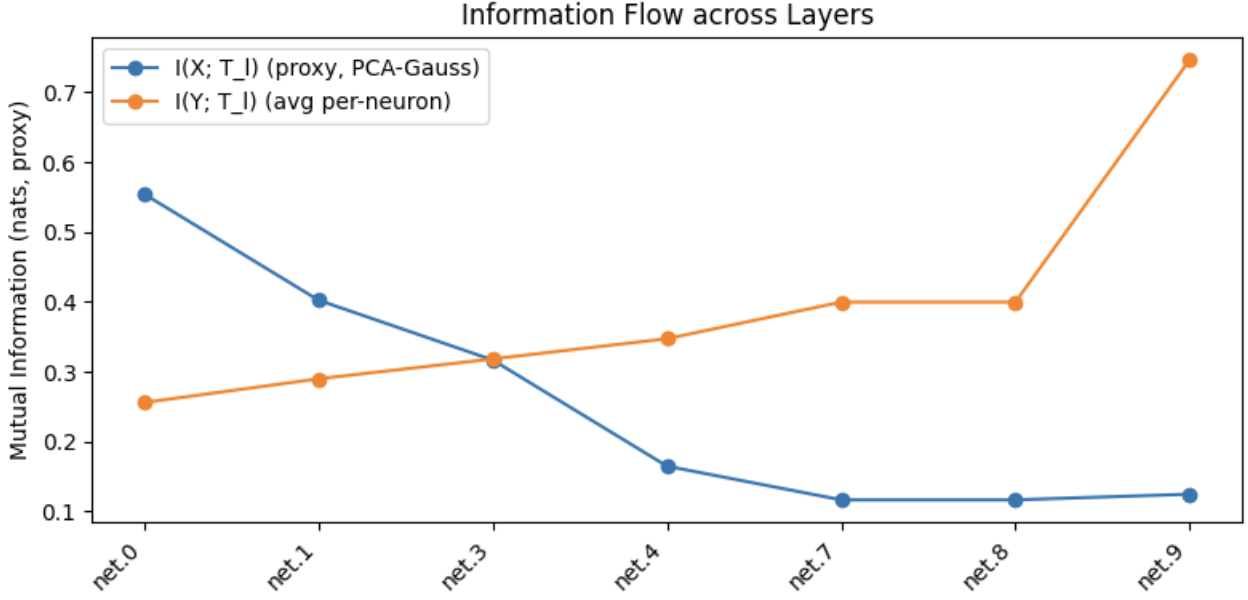}
    \caption{Information flow across layers for the CNN.}
    \label{fig:placeholder}
\end{figure}

Across all architectures, there is a consistent pattern:

\begin{itemize}
    \item Early layers preserve input information.
    \item Intermediate layers compress irrelevant features.
    \item Deeper layers retain task-relevant structure for classification.
\end{itemize}

These results suggest that neural networks can reorganize representations such that task-relevant structure becomes more accessible to the chosen estimator or classifier. However, this shouldn't be interpreted as true mutual information increasing through the network; rather, the empirical proxies reveal changes in representation geometry and estimator accessibility. The broader point remains that the learned network does not enforce a comparison principle, contraction property, or stability guarantee under perturbations.

\subsection{Information Flow in Sequential Decision Models}

To examine whether similar behavior arises in sequential settings, a reinforcement learning environment is considered. The system is modeled as a Markov decision process (MDP), where states $X_t$, actions $A_t$, and rewards $R_t$ evolve over time. A Deep Q-Network approximates the action-value function:
\[
Q_\theta(X_t,a) \approx \mathbb{E}\!\Bigg[\sum_{k=0}^\infty \gamma^k R_{t+1+k} \,\Big|\, X_t, A_t = a\Bigg]
\]

Internal activations $T_{\cdot,l}(X_t)$ define representations across layers. The network must encode latent task variables $Y_i$ (e.g., position, velocity) necessary for decision-making.

Information flow is measured via:
\[
I(X; T_{\cdot,l}), \quad I(A; T_{\cdot,l}), \quad I(Y_i; T_{j,l})
\]

Segregation is captured by unit-level specialization:
\[
S_{j,l} = \max_i I(Y_i; T_{j,l}) - \frac{1}{k-1} \sum_{i' \neq \arg\max_i} I(Y_{i'}; T_{j,l})
\]

Integration is reflected in layers where representations simultaneously encode multiple latent variables or remain predictive of actions despite compression.

The CartPole-v1 environment provides a simple sequential setting. In this game, the environment state $X_t \in \mathbb{R}^4$ at time $t$ consists of four variables: cart position, cart velocity, pole angle, and pole angular velocity. At each time step, the agent chooses a discrete action $A_t \in \{0,1\}$ corresponding to applying a horizontal force either left or right.  The agent receives reward $r_t = 1$ for every time step the pole remains upright and the cart stays within bounds; likewise, the game ends when the pole falls or a maximum time horizon (500 steps) is reached. In this framework, $X$ corresponds to the raw state variables and the task-relevant “hidden variables” $Y_i$ are simple functions of these components (e.g.\ whether the pole is falling left or right, whether the cart is near the boundary), which must be implicitly encoded by the network in order to choose appropriate actions.A two-layer MLP policy is trained using REINFORCE. After training, activations are recorded across $\sim 5{,}500$ state-action pairs.

Figure~12 shows the training curve over 200 episodes. Returns are initially low, then improve as the policy discovers stabilizing behavior, though fluctuations remain due to the high variance of on-policy training. The purpose of this figure is not to optimize CartPole performance, but to confirm that the learned network reaches a regime where its internal representations can be meaningfully analyzed.

\begin{figure}[H]
        \centering
        \includegraphics[width=0.75\linewidth]{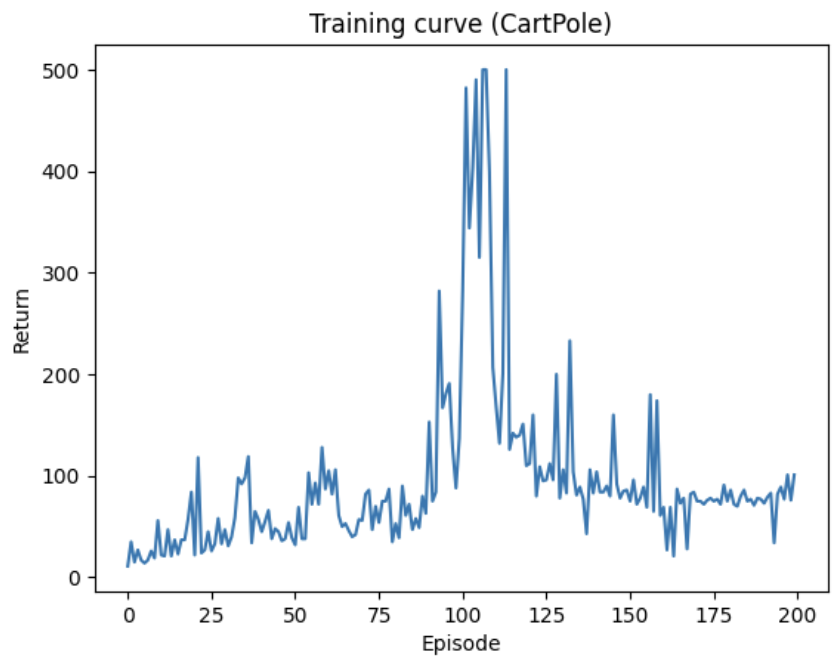}
        \caption{Training curve for the CartPole policy network}
    \label{fig:placeholder}
\end{figure}

To connect this control setting to the information-flow framework above, internal activations from the trained policy network are recorded and mutual information is estimated at each layer. After training, the learned policy is run in greedy mode (choosing $\arg\max_a \pi_\theta(a\mid x)$) for 50 evaluation episodes, yielding $N \approx 5{,}500$ state--action pairs. For each visited state $X_t$, the corresponding layer activations $T_{\ell}(X_t)$ are logged for all hidden and output layers, along with the realized actions $Y_t := A_t$.

Using the same information-flow methodology as in the feedforward setting, $I(X;T_\ell)$ is approximated by first computing a low-dimensional PCA representation $Z$ of the input states $X$. For each neuron in layer $\ell$, the squared multiple correlation $R^2$ between its activations and the top principal components of $X$ is then computed. Under a Gaussian approximation for the joint distribution of activations and inputs, mutual information admits the closed-form estimate
\[
\hat I(X;T_\ell) \approx -\tfrac{1}{2} \log(1-R^2)
\]
which is averaged across neurons in layer $\ell$. To approximate $I(Y;T_\ell)$, each neuron’s activations are discretized into quantile-based bins, and empirical mutual information with the discrete action labels $Y_t$ is computed:
\[
\hat I(Y;T_\ell^j)
=
\sum_{y,t} p(y,t)\log\frac{p(y,t)}{p(y)p(t)}
\]
again averaging over neurons $j$ in layer $\ell$. The resulting information profile for CartPole (Figure 13) exhibits a characteristic \emph{down--up--down--sharp-up} shape for both $\hat I(X;T_\ell)$ and $\hat I(Y;T_\ell)$. 

\begin{figure}[H]
        \centering
        \includegraphics[width=0.75\linewidth]{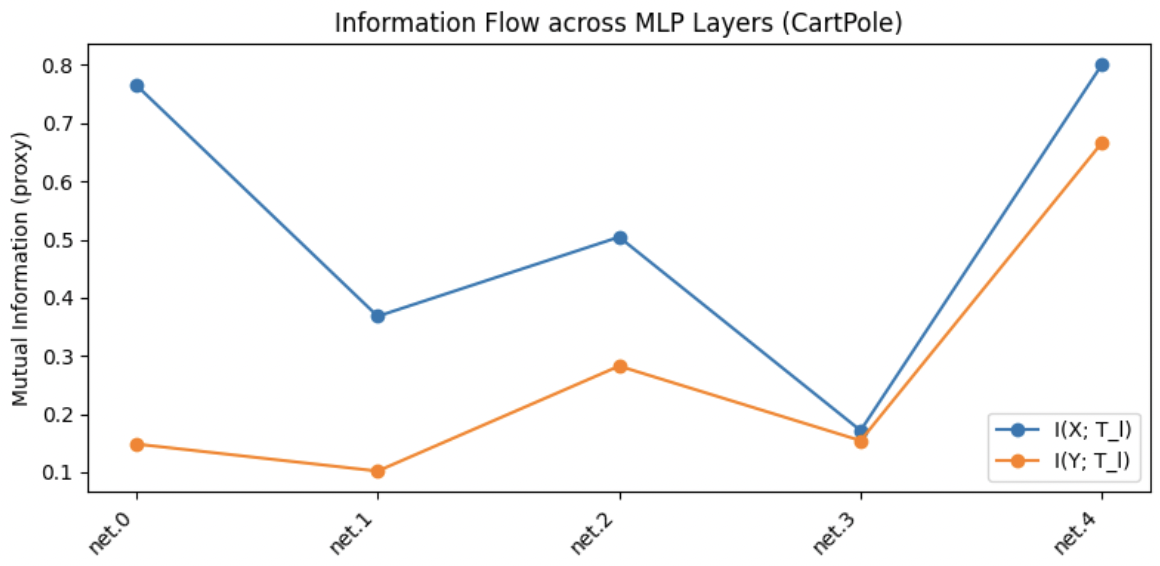}
        \caption{Information flow across layers of the CartPole policy network}
    \label{fig:placeholder}
\end{figure}

\smallskip
\noindent\emph{(i) Input-adjacent layer (net.0):}
$\hat I(X;T_\ell)$ is high while $\hat I(Y;T_\ell)$ is low. This layer preserves a rich representation of the physical game state (pole angle, cart velocity, cart position, angular velocity), but carries essentially no information about the action $Y$. At this stage, the network has not yet determined which state dimensions will be relevant for behavioral choice.

\smallskip
\noindent\emph{(ii) Early hidden layer (net.1):}
Both information measures dip. Here, the network begins discarding task-irrelevant components of $X$. Since the model is still evaluating the information, the policy still carries little information about $Y$.

\smallskip
\noindent\emph{(iii) Mid-level layer (net.2):}
Both estimated curves rise at this stage. Here, the network has begun to segregate the dimensions of $X$ that are relevant and organizes them into a form that is more predictive of the action under the chosen estimator. As a result, $\widehat I(X;T_\ell)$ and $\widehat I(Y;T_\ell)$ increase as empirical proxy values, reflecting increased estimator accessibility rather than an increase in exact mutual information. This corresponds to a \emph{feature integration} stage. Informationally, the representation transitions from just reflecting the physical state to shaping that state into a task-relevant signal for action selection; in more intuitive terms, the model moves from passively viewing the game to actively noticing which features matter (e.g., whether the pole is falling left or right) and integrating these observations into its eventual choice of action.

\smallskip
\noindent\emph{(iv) Bottleneck layer (net.3):}
A second dip appears. Here, $\hat I(X;T_\ell)$ decreases as the network removes information it now considers irrelevant from the previous integration stage, retaining only what is useful for distinguishing left vs.\ right corrections. Because the representation becomes more compact, $\hat I(Y;T_\ell)$ also dips slightly, reflecting that the action-relevant signal has not yet been fully committed to a discrete choice. This stage resembles a \emph{decision bottleneck}: the representation is small, efficient, and nearly sufficient for action selection, but the final commitment has not yet occurred.

\smallskip
\noindent\emph{(v) Output layer (net.4):}
Both empirical proxy values $\widehat I(X;T_\ell)$ and $\widehat I(Y;T_\ell)$ jump sharply at the output layer. At this stage, the network is no longer primarily filtering or compressing the representation; it is committing to an action. The rise in $\widehat I(Y;T_\ell)$ should be interpreted as the output representation becoming more directly aligned with the action labels used by the estimator. It does not imply that true mutual information has increased through the network.

\smallskip
Taken together, these experiments demonstrate that neural networks can reorganize representations so that task-relevant structure becomes more visible to empirical estimators and downstream classifiers. However, because the measured increases are proxy-based, they should not be interpreted as increases in exact mutual information. The relevant conclusion is instead that learned representations evolve without an explicit stability mechanism such as contraction or order preservation, thus motivating the introduction of the hybrid reconstruction operator in the main text, where nonlinear diffusion provides a mathematically controlled evolution that suppresses perturbations while preserving structure. In this regard, the empirical observations here serve as a baseline against which the stability properties of the proposed framework can be understood.

\section{Instability of Learned Reconstructions}

The purpose of this appendix is to explicitly connect the statistical variability observed in stochastic optimization (Appendix A) and the representation dynamics of neural networks (Appendix B) to the central modeling choice of this work: the introduction of a diffusion-based operator to stabilize learned reconstructions.

\subsection{Variability in Learned Parameters and Outputs}

As shown in Appendix A, SGD produces a sequence of iterates $(x_1, x_2, \dots)$ that converge only in distribution to a limiting parameter $x^*$, with asymptotic fluctuations governed by a covariance matrix. Even when averaged estimators converge, finite-sample realizations exhibit variability across runs and training epochs.

Let $G_{\theta_k}$ denote the learned reconstruction operator at epoch $k$, with parameters $\theta_k$ obtained via stochastic optimization. Then the reconstruction output
\[
u_k = G_{\theta_k}(u_{\text{in}})
\]
is itself a random variable induced by optimization noise.

Thus, even for fixed input $u_{\text{in}}$, the mapping
\[
k \mapsto G_{\theta_k}(u_{\text{in}})
\]
isn't constant and may exhibit significant variation across training checkpoints.

\subsection{Lack of Structural Control in Representation Learning}

Appendix B demonstrates that neural network representations evolve across layers without satisfying any comparison principle or contraction property. In particular, there is no guarantee that small perturbations in the input or parameters lead to controlled changes in the output. There is no general bound
\[
\|G_{\theta_k}(u_0) - G_{\theta_k}(v_0)\|_\infty \leq C \|u_0 - v_0\|_\infty
\]
with a uniform constant $C$ independent of $k$, nor any guarantee that
\[
\|G_{\theta_k}(u_{\text{in}}) - G_{\theta_{k'}}(u_{\text{in}})\|_\infty
\]
is small for nearby checkpoints $k, k'$.

This lack of control is intrinsic to learned mappings trained via stochastic optimization and is not addressed by standard architectures.

\subsection{Consequence for Reconstruction Pipelines}

\begin{itemize}
    \item Stochastic training induces variability in model parameters
    \item this variability propagates through the learned operator
    \item the learned operator provides no structural mechanism to suppress this variability
\end{itemize}

As a result, downstream quantities (such as AUC) may fluctuate across epochs even when the underlying task is unchanged.

\subsection{Operator-Level Stabilization via Nonlinear Diffusion}

The hybrid reconstruction operator
\[
H_t(u_0) = T_t(G_{\theta}(u_0))
\]
introduces a second stage $T_t$ with known structural properties at the continuous level:
\begin{itemize}
    \item order preservation
    \item comparison principle
    \item non-expansiveness in $L^\infty$
\end{itemize}

While the discrete implementation used in this work does not establish these properties rigorously, it is constructed as a finite approximation of such an evolution.

This leads to the following interpretation:
\begin{itemize}
    \item variability enters through $G_\theta$,
    \item the diffusion operator $T_t$ acts as a post-processing map,
    \item this map suppresses high-frequency perturbations and enforces a controlled evolution.
\end{itemize}

Thus, the role of the diffusion step is not to improve representation learning, but to impose stability at the level of the operator.

\subsection{Connection to Empirical Observations}

The reduction in epoch-to-epoch variability observed is consistent with this interpretation. While a formal bound on perturbation amplification is not established, the results align with the hypothesis that operator-level structure can mitigate instability arising from stochastic training.

\section{Numerical Implementation of the Diffusion Step}
\begin{lstlisting}[language=Python]
import numpy as np

def gradient(u):
    """Forward finite differences with Neumann boundary conditions"""
    ux = np.zeros_like(u)
    uy = np.zeros_like(u)
    
    ux[:, :-1] = u[:, 1:] - u[:, :-1]
    uy[:-1, :] = u[1:, :] - u[:-1, :]
    
    return ux, uy

def divergence(px, py):
    """Backward finite differences (discrete divergence)"""
    div = np.zeros_like(px)
    
    div[:, :-1] -= px[:, :-1]
    div[:, 1:]  += px[:, :-1]
    
    div[:-1, :] -= py[:-1, :]
    div[1:, :]  += py[:-1, :]
    
    return div

def diffusivity(grad_mag, kappa):
    """Edge-stopping function: 1 / (1 + (s/kappa)^2)"""
    return 1.0 / (1.0 + (grad_mag / kappa)**2)

def diffusion_step(u, dt, kappa):
    ux, uy = gradient(u)
    grad_mag = np.sqrt(ux**2 + uy**2)
    
    g = diffusivity(grad_mag, kappa)
    
    px = g * ux
    py = g * uy
    
    return u + dt * divergence(px, py)

def refine_slice(u0, K, dt, kappa):
    """Apply K Euler steps to one 2D slice"""
    u = u0.astype(np.float32, copy=True)
    for _ in range(K):
        u = diffusion_step(u, dt, kappa)
    return u

def refine_volume(u0, K, dt, kappa):
    """Apply diffusion slice-wise to 2D/3D input"""
    arr = np.asarray(u0)
    
    if arr.ndim == 2:
        return refine_slice(arr, K, dt, kappa)
    
    if arr.ndim == 4 and arr.shape[-1] == 1:
        arr = arr[..., 0]
    
    if arr.ndim != 3:
        raise ValueError(f"Expected (H,W), (D,H,W), or (D,H,W,1), got {arr.shape}")
    
    D, H, W = arr.shape
    out = np.empty((D, H, W), dtype=np.float32)
    
    for d in range(D):
        out[d] = refine_slice(arr[d], K, dt, kappa)
    
    return out
\end{lstlisting}

\end{document}